\documentclass{article}

\PassOptionsToPackage{square,numbers}{natbib}

\usepackage[final]{neurips_2022}

\usepackage{amsthm,amsmath,bbm,amsfonts,amssymb}
\usepackage{dsfont} 
\usepackage{mathtools}
\usepackage{commath}
\mathtoolsset{showonlyrefs=false}

\usepackage{xspace}
\usepackage[T1]{fontenc}
\usepackage[utf8]{inputenc}
\usepackage[usenames,dvipsnames]{xcolor} 

\usepackage{hyperref,url}
\usepackage{cleveref} 

\usepackage[usenames,dvipsnames]{xcolor} 

\usepackage{wrapfig}
\usepackage[export]{adjustbox}
\pdfoutput=1  

\usepackage{algorithm}

\usepackage{graphicx,tabularx}
\usepackage{multirow,hhline}
\usepackage{booktabs}

\usepackage{capt-of}

\usepackage{enumitem}

\allowdisplaybreaks 

\usepackage{pbox} 

\usepackage[export]{adjustbox}

\usepackage{pifont}
\usepackage{framed}
\usepackage[most]{tcolorbox}
\definecolor{kjgray}{rgb}{.7,.7,.7}

\newtheoremstyle{kjstyle}
{1ex} 
{\topsep} 
{\itshape} 
{} 
{\bfseries} 
{.} 
{.5em} 
{} 

\newtheoremstyle{kjstylenoitalic}
{1ex} 
{\topsep} 
{} 
{} 
{\bfseries} 
{.} 
{.5em} 
{} 

\tcbset{kjboxstyle/.style={title={},breakable,colback=white,enhanced jigsaw,boxrule=1.3pt,sharp corners,colframe=kjgray,boxsep=0pt,coltitle={black},attach title to upper={},left=.8ex,bottom=.4em}}    
\theoremstyle{kjstyle}
\AtBeginEnvironment{thm}{\begin{tcolorbox}[kjboxstyle]}%
  \AtEndEnvironment{thm}{\end{tcolorbox}}%
\theoremstyle{kjstyle}
\AtBeginEnvironment{lem}{\begin{tcolorbox}[kjboxstyle]}%
  \AtEndEnvironment{lem}{\end{tcolorbox}}%
\theoremstyle{kjstyle}
\AtBeginEnvironment{prop}{\begin{tcolorbox}[kjboxstyle]}%
  \AtEndEnvironment{prop}{\end{tcolorbox}}%
\theoremstyle{kjstyle}
\AtBeginEnvironment{cor}{\begin{tcolorbox}[kjboxstyle]}%
  \AtEndEnvironment{cor}{\end{tcolorbox}}%

\theoremstyle{kjstyle}
\AtBeginEnvironment{defn}{\begin{tcolorbox}[kjboxstyle]}%
  \AtEndEnvironment{defn}{\end{tcolorbox}}%

\theoremstyle{kjstyle}
\AtBeginEnvironment{ass}{\begin{tcolorbox}[kjboxstyle]}%
  \AtEndEnvironment{ass}{\end{tcolorbox}}%

\theoremstyle{kjstyle}
\AtBeginEnvironment{question}{\begin{tcolorbox}[kjboxstyle]}%
  \AtEndEnvironment{question}{\end{tcolorbox}}%

\theoremstyle{kjstylenoitalic}
\AtBeginEnvironment{remark}{\begin{tcolorbox}[kjboxstyle]}%
  \AtEndEnvironment{remark}{\end{tcolorbox}}%

\usepackage{mdframed}
\usepackage{lipsum}
\definecolor{kjgray}{rgb}{.7,.7,.7}
\makeatletter



\makeatletter
\renewcommand{\paragraph}{%
  \@startsection{paragraph}{4}%
  {\z@}{0.8em \@plus 1ex \@minus .2ex}{-1em}%
  {\normalfont\normalsize\bfseries}%
}
\makeatother

\newcounter{textcnt}

\newcommand\addtext[1]{%
  \stepcounter{textcnt}%
  \csgdef{text\thetextcnt}{#1}}

\newcounter{colnum}

\newcounter{Jidx}
\newcommand\dsymhelper[2]{
  \addtext{\hyperlink{#1}{#2}}%
  \blue{\hypertarget{#1}{#2}}%
}
\newcommand\dsym[1]{
  \stepcounter{Jidx}
  \xdef\tmpname{Jsym.\theJidx}
  \expandafter\dsymhelper\expandafter{\tmpname}{#1}
} 

\def\ddefloop#1{\ifx\ddefloop#1\else\ddef{#1}\expandafter\ddefloop\fi}
\def\ddef#1{\expandafter\def\csname c#1\endcsname{\ensuremath{\mathcal{#1}}}}
\ddefloop ABCDEFGHIJKLMNOPQRSTUVWXYZ\ddefloop
\def\ddef#1{\expandafter\def\csname b#1\endcsname{\ensuremath{{\boldsymbol{#1}}}}}
\ddefloop ABCDEFGHIJKLMNOPQRSUVWXYZabcdeghijklmnopqrtsuvwxyz\ddefloop  
\def\ddef#1{\expandafter\def\csname h#1\endcsname{\ensuremath{\widehat{#1}}}}
\ddefloop ABCDEFGHIJKLMNOPQRSTUVWXYZabcdefghijklmnopqrsuvwxyz\ddefloop 
\def\ddef#1{\expandafter\def\csname hc#1\endcsname{\ensuremath{\widehat{\mathcal{#1}}}}}
\ddefloop ABCDEFGHIJKLMNOPQRSTUVWXYZ\ddefloop
\def\ddef#1{\expandafter\def\csname t#1\endcsname{\ensuremath{\widetilde{#1}}}}
\ddefloop ABCDEFGHIJKLMNOPQRSTUVWXYZabcdefghijklmnopqrtsuvwxyz\ddefloop
\def\ddef#1{\expandafter\def\csname r#1\endcsname{\ensuremath{\mathring{#1}}}}
\ddefloop ABCDEFGHIJKLMNOPQRSTUVWXYZabcdefghijklmnopqrtsuvwxyz\ddefloop
\def\ddef#1{\expandafter\def\csname tc#1\endcsname{\ensuremath{\widetilde{\mathcal{#1}}}}}
\ddefloop ABCDEFGHIJKLMNOPQRSTUVWXYZ\ddefloop

\DeclareMathOperator*{\argmax}{arg~max}

\DeclareMathOperator{\EE}{\mathbb{E}}

\DeclareMathOperator{\PP}{\mathbb{P}}

\DeclareMathOperator{\one}{\mathds{1}\hspace{-0.0em}}
\DeclareMathOperator{\Reg}{{\text{Reg}}}

\def\RR{{\mathbb{R}}}

\newcommand{\sr}[2]{ {\mathop{}\stackrel{#1}{#2}}\mathop{}\, }

\newcommand{\fr}[2]{ { \frac{#1}{#2} }}

\def\eps{\ensuremath{\epsilon}\xspace}
\def\hS{\ensuremath{\hat{S}}\xspace}

\def\T{\ensuremath{\top}}  
\def\sig{\ensuremath{\sigma}\xspace}

\def\eps{\ensuremath{\epsilon}\xspace}

\def\supp{\ensuremath{\mathrm{supp}}\xspace}

\def\sm{{\ensuremath{\setminus}\xspace} }

\def\sym{{\ensuremath{\text{Sym}}\xspace}}

\makeatletter
\newcommand{\vast}{\bBigg@{3}}
\newcommand{\Vast}{\bBigg@{4}}
\makeatother

\def\cX{\ensuremath{\mathcal{X}}\xspace} 
 
\def\la{{\langle}}
\def\ra{{\rangle}}
\def\det{\ensuremath{\mbox{det}}}

\def\tilth{\ensuremath{\tilde{\boldsymbol\theta}}\xspace}

\def\cC{\ensuremath{\mathcal{C}}\xspace}

\def\cU{\ensuremath{\mathcal{U}}\xspace}

\def\vec{\text{vec}}

\def\vec{\ensuremath{\text{vec}}}

\def\hS{\ensuremath{\what{\S}}\xspace}

\def\hS{\ensuremath{\hat S}}

\def\th{{\ensuremath{\theta}}}

\def\cH{{\ensuremath{\mathcal{H}}}}

\def\cd{\cdot}

\renewcommand{\vec}{{\mathsf{vec}}}

\newcommand\kmax[1]{\mathop{#1\hyphen\max}}
\def\hyphen{{\text{-\hspace{-.06em}}}}

\def\tmin{{\text{min}}}
\def\Cmin{{\cC_\tmin}}
\def\Lasso{\ensuremath{\normalfont{\text{Lasso}}}}

\def\thp{{\theta'}}
\def\th{\theta}
\def\hatsr{u}

\def\tilth{{\tilde{\theta}}}
\def\indic{\one}

\def\kap{{\kappa}}
\def\tsty{\textstyle}
\def\wed{\wedge}
\newcommand{\evt}[1]{\envert{#1}}

\def\GIANTKL{{\fr{4s^3n}{d}+  77 s \kap^2 \cT(\cH; a)}}
\def\GIANTKLtau{{\fr{4s^3n}{d}+  77 s \kap^2 \tau}}

\def\alg{\ensuremath{\mathsf{Alg}}\xspace}

\def\alga{\ensuremath{\widetilde{\mathsf{Alg}}}\xspace}
\def\algapi{\ensuremath{\widetilde{\mathsf{Alg}}\pi}\xspace}
\def\algap{\ensuremath{\widetilde{\mathsf{Alg}}\mathsf{P}}\xspace}
\newcommand{\ind}{\mathbbm{1}}
\newcommand{\legal}{\text{legal}}

\usepackage[utf8]{inputenc} 
\usepackage[T1]{fontenc}    
\usepackage{hyperref}       
\usepackage{url}            
\usepackage{booktabs}       
\usepackage{amsfonts}       
\usepackage{nicefrac}       
\usepackage{microtype}      
\usepackage{xcolor}         

\let\vec\undefined 
\usepackage{MnSymbol} 

\usepackage{enumitem}
\setlist[itemize]{topsep=.5pt,itemsep=0pt,parsep=2pt,leftmargin=2em}
\setlist[enumerate]{topsep=.5pt,itemsep=0pt,parsep=2pt,leftmargin=2em}

\usepackage{algorithmic}

\usepackage{amsthm}
\newtheorem{theorem}{Theorem}
\newtheorem{proposition}{Proposition}
\newtheorem{corollary}{Corollary}
\newtheorem{assumption}{Assumption}
\newtheorem{definition}{Definition}
\newtheorem{lemma}{Lemma}
\newtheorem{remarks}{Remark}
\newtheorem{claim}{Claim}
\newcommand{\blue}[1]{{\color[rgb]{.3,.5,1}#1}}

\newcommand{\cz}[1]{}
\newcommand{\chicheng}{\cz}
\newcommand{\ja}[1]{}

\usepackage[normalem]{ulem}
\newcommand{\edit}[2]{{\color{blue}{\sout{#1}}}{\color{red}{#2}}}

\renewcommand{\blue}[1]{{#1}}

\usepackage[toc,page,header]{appendix}
\usepackage{minitoc}
\usepackage{silence}
\usepackage{xspace}
\newcommand\inner[2]{{\langle#1,#2\rangle}}
\newcommand\popart{\ensuremath{\textsc{PopArt}}\xspace}
\newcommand\wpopart{\ensuremath{\textsc{Warm-PopArt}}\xspace}
\newcommand\cova{X}
\newcommand\resp{Y}
\newcommand\B{\mathsf{B}}
\newcommand\Bpi{\mathsf{B}\pi}
\newcommand\BP{\mathsf{BP}}
\newcommand\var{\textrm{Var}}
\newcommand\polylog{\mathrm{polylog}}

\usepackage{hyperref}

\renewcommand\paragraph[1]{\noindent\textbf{#1}}

\usepackage{cases}

\title{PopArt: Efficient Sparse Regression and  Experimental Design for Optimal Sparse Linear Bandits}

\author{%
  Kyoungseok Jang \\
  University of Arizona\\
  \texttt{ksajks@arizona.edu} \\
  \And
  Chicheng Zhang \\
  University of Arizona \\
  \texttt{chichengz@cs.arizona.edu } \\
  \And
  Kwang-Sung Jun\\
  University of Arizona\\
  \texttt{kjun@cs.arizona.edu} \\
}

\begin{document}
\doparttoc 
\faketableofcontents 


\setlength{\abovedisplayskip}{5pt}
\setlength{\belowdisplayskip}{5pt}
\setlength{\abovedisplayshortskip}{5pt}
\setlength{\belowdisplayshortskip}{5pt}


\maketitle

\begin{abstract}
In sparse linear bandits, a learning agent sequentially selects an action and receive reward feedback, and the reward function depends linearly on a few coordinates of the covariates of the actions. 
This has applications in many real-world sequential decision making problems. 
In this paper, we propose a simple and computationally efficient sparse linear estimation method called \popart that enjoys a tighter $\ell_1$ recovery guarantee compared to Lasso (Tibshirani, 1996) in many problems. 
Our bound naturally motivates an experimental design criterion that is convex and thus computationally efficient to solve.
Based on our novel estimator and design criterion, we derive sparse linear bandit algorithms that enjoy improved regret upper bounds upon the state of the art (Hao et al., 2020), especially w.r.t. the geometry of the given action set.
Finally, we prove a matching lower bound for sparse linear bandits in the data-poor regime, which closes the gap between upper and lower bounds in prior work.
\end{abstract}


\textfloatsep=.6em

\section{Introduction}
\label{sec:intro}

In many modern science and engineering applications, high-dimensional data naturally emerges, where the number of features significantly outnumber the number of samples. 
In gene microarray analysis for cancer prediction~\cite{ramaswamy2001multiclass}, for example, tens of thousands of genes expression data are measured per patient, far exceeding the number of patients. Such practical settings motivate the study of high-dimensional statistics, where certain structures of the data are exploited to make statistical inference possible. One representative example is sparse linear models~\cite{hastie2015statistical}, where we assume that a linear regression task's underlying predictor depends only on a small subset of the input features.


On the other hand, online learning with bandit feedback, due to its practicality in many applications such as online news recommendations~\cite{li10acontextual} or clinical trials~\cite{liao2016sample, woodroofe1979one}, has attracted a surge of research interests. 
Of particular interest is linear bandits, where in $n$ rounds, the learner repeatedly takes an action $A_t$ (e.g., some feature representation of a product or a medicine) from a set of available actions $\cA \subset \RR^d$ and receives a reward $r_t = \inner{\theta^*}{A_t} + \eta_t$ as feedback where  $\eta_t \in \RR$ is an independent zero-mean, $\sigma$-sub-Gaussian noise. 
Sparsity structure is abundant in linear bandit applications: for example, customers' interests on a product depend only on a number of its key specs; the effectiveness of a medicine only depends on a number of key medicinal properties, which means that the unknown parameter $\theta^*$ is sparse; i.e., it has a small number of nonzero entries.

Early studies~\cite{ay12online,carpentier2012bandit,lattimore2015linear} on sparse linear bandits have revealed that leveraging sparsity assumptions yields bandit algorithms with lower regret than those provided by full-dimensional linear bandit algorithms~\cite{abe99associative,auer02using,dani08stochastic,abbasi2011improved}. 
However, most existing studies either rely on a particular arm set (e.g., a norm ball), which is unrealistic in many applications, or use computationally intractable algorithms.
If we consider an arbitrary arm set, however, the optimal worst-case regret is ${\Theta}(\sqrt{sdn})$ where $s$ is the sparsity level of $\th^*$, which means that as long as $n = O(sd)$, there exists an instance for which the algorithm suffers a linear regret~\cite{lattimore18bandit}.
This is in stark contrast to supervised learning where it is possible to enjoy nontrivial prediction error bounds for $n = o(d)$~\cite{foster94risk}.
This motivates a natural research question: Can we develop computationally efficient {sparse linear bandit} algorithms that allow a generic arm set  yet enjoy nonvacuous bounds in the data-poor regime by exploiting problem-dependent characteristics?

The seminal work of~\citet{hao2020high} provides a positive answer to this question.
They propose algorithms that enjoy nonvacuous regret bounds with an arbitrary arm set in the data poor regime using Lasso.
Specifically, they have obtained a regret bound of $\tilde O({\Cmin}^{-2/3} s^{2/3}n^{2/3})$ where ${\Cmin}$ is an arm-set-dependent quantity. 
However, their work still left a few open problems.
First, their regret upper bound does not match with their lower bound $\Omega({\Cmin}^{-1/3} s^{1/3}n^{2/3})$.
Second, it is not clear if ${\Cmin}$ is the right problem-dependent constant that captures the geometry of the arm set.

\begin{table}[t]
\begin{center}
\begin{small}
\begin{sc}
\begin{tabular}{lcccr}
\toprule
 & Regret Bound & Data-poor & Assumptions \\
\midrule
\citet{hao2020high} & $\tilde{O}(s^{2/3}\mathcal{C}_{\min}^{-2/3}n^{2/3})$ & \ding{51} & $\mathcal{A}$ spans $\mathbb{R}^d$ \\
\citet{hao2020high} & ${\Omega(s^{1/3}\kappa^{-2/3}n^{2/3})}$ &\ding{51} & $\cA$ spans $\RR^d$ \\
Algorithm~\ref{alg:etc-sparse} (Ours)  & $\tilde{O}(s^{2/3} H_*^{2/3}n^{2/3})$ & \ding{51} & $\mathcal{A}$ spans $\mathbb{R}^d$\\
Theorem~\ref{thm:lower} (Ours) & $\Omega(s^{2/3}\kappa^{-2/3}n^{2/3})$ &\ding{51} & $\cA$ spans $\RR^d$\\
\hline
\citet{hao2020high} & $\tilde{O}(\sqrt{\mathcal{C}_{\min}^{-1}sn})$ & \ding{55} & $\mathcal{A}$ spans $\mathbb{R}^d$, Min. Signal \\
Algorithm~\ref{alg:phase-elim} (Ours) & $\tilde{O}(\sqrt{sn})$ & \ding{55} & $\mathcal{A}$ spans $\mathbb{R}^d$, Min. Signal\\
\bottomrule
\end{tabular}
\end{sc}
\caption{
  Regret bounds of our work and the prior art where $s$, $d$, $n$ are the sparsity level, the feature dimension, and the number of rounds, respectively. 
  The quantities $\mathcal{C}_{\min}$ and $H_*^2$ are the constants that captures the geometry of the action set (see Eq.~\eqref{def: Cmin} and~\eqref{def: H2}), and $\kappa$ is a parameter for a specific family of arm sets that satisfies $ \kappa^{-2} =  \Theta( \mathcal{C}_{\min}^{-1}) = \Theta( H^2_* )$. 
  In general, $ H_*^2 \leq \mathcal{C}_{\min}^{-1} \le \mathcal{C}_{\min}^{-2}$ (Propositon~\ref{prop:H2 vs Cmin}).
}
\label{table: results}
\end{small}
\end{center}
\end{table}

In this paper, we make significant progress in high-dimensional linear regression and sparse linear bandits, which resolves or partly answers the aforementioned open problems.

\textbf{First} (Section~\ref{sec:popart}), we propose a novel and computationally efficient estimator called \popart (POPulation covariance regression with hARd Thresholding) that enjoys a tighter $\ell_1$ norm recovery bound than the de facto standard sparse linear regression method Lasso in many problems.
Motivated by the $\ell_1$ norm recovery bound of \popart, we develop a computationally-tractable design of experiment objective for finding the sampling distribution that minimize the error bound of \popart, which is useful in settings where we have control  on the sampling distribution (such as compressed sensing).
Our design of experiments results in an $\ell_1$ norm error bound that depends on the measurement set dependent quantity denoted by $H_*^{2}$ (see Eq.~\eqref{def: H2} for precise definition) that is provably better than ${\Cmin}^{-1}$ that appears in the $\ell_1$ norm error bound used in \citet{hao2020high}, thus leading to an improved planning method for {sparse linear} bandits.
\textbf{Second} (Section~\ref{sec:bandits}), Using \popart, we design new algorithms for the sparse linear bandit problem, and improve the regret upper bound of prior work~\cite{hao2020high}; see Table~\ref{table: results} for the summary.
\textbf{Third} (Section~\ref{sec:lower-bound}), We prove a matching lower bound in data-poor regime, showing that the regret rate obtained by our algorithm is optimal. 
The key insight in our lower bound is a novel application of the algorithmic symmetrization technique \cite{simchowitz2017simulator}. {Unlike the conjecture of \citet[Remark 4.5]{hao2020high}, the improvable part was not the algorithm but the lower bound for sparsity $s$.}

We empirically verify our theoretical findings in Section~\ref{sec:expr} where \popart shows a favorable performance over Lasso.
Finally, we conclude our paper with future research enabled by \popart in Section~\ref{sec:conclusion}.
Due to space constraints, we discuss related work in Appendix~\ref{sec:related} but closely related studies are discussed in depth throughout the paper.

\vspace{-7pt}
\section{Problem Definition and Preliminaries} \label{sec: prelim}
\vspace{-6pt}

\paragraph{Sparse linear bandits.} We study the sparse linear bandit learning setting, where the learner is given access to an action space ${\cA} \subset \{a\in\RR^d: \|a\|_\infty  \le 1\}$, and repeatedly interacts with the environment as follows: at each round $t = 1,\ldots,n$, the learner chooses some action $A_t \in \cA$, and receives reward feedback $r_t = \inner{\theta^*}{A_t} + \eta_t$, where ${\theta^*} \in \RR^d$ is the underlying reward predictor, and ${\eta_t}$ is an independent zero-mean $\sigma$-subgaussian noise.
We assume that $\theta^*$ is $s$-sparse; that is, it has at most $s$ nonzero entries. 
The goal of the learner is to minimize its pseudo-regret defined as 
\[
\Reg(n) = n \max_{a \in \cA} \inner{\theta^*}{a} - \sum_{t=1}^n \inner{\theta^*}{A_t}.
\]
\ja{Change $\Reg(n)$ to $\Reg_n$, especially in main body}
\paragraph{Experimental design for linear regression.} In the experimental design for linear regression problem, one has a pool of unlabeled examples $\cX$, and some underlying predictor $\theta^*$ to be learned. 
Querying the label of $x$, i.e. conducting experiment $x$, reveals a random label $y = \inner{\theta^*}{x} + \eta$ associated with it, where $\eta$ is a zero mean noise random variable. The goal is to accurately estimate $\theta^*$, while using as few queries $x$ as possible. 




\begin{definition}(Population covariance matrix $Q$) 
Let $\mathcal{P}({\cX})$ be the space of probability measures over $\mathcal{X}$ with the Borel $\sigma$-algebra, and define the population covariance matrix for the distribution $\mu \in \mathcal{P}(\cX)$ as follows:
\begin{equation}
Q(\mu):=\int_{a \in \mathcal{X}} a a^\top d \mu(a)
\end{equation}
Classical approaches for experimental design focus on finding a distribution $\mu$ such that its induced population covariance matrix $Q(\mu)$ has properties amenable for building a low-error estimator, such as D-, A-, G-optimality~\cite{fedorov2013theory}.

\textbf{Compatibility condition for Lasso.~}
For a positive definite matrix $\Sigma \in \RR^{d\times d}$ and a sparsity level $s\in[d] := \{1,\ldots,d\}$, 
we define its compatibility constant $\phi_0^2 (\Sigma, s)$~\cite{bv11}
as follows:
\begin{equation}
\phi_0^2 (\Sigma, s):=\min_{S\subseteq[d]: |S| = s}~ \min_{v: \| v_{S} \|_1 \leq 3\| v_{-S} \|_1}\frac{s v^\top \Sigma v}{\| v_S \|_1^2},\label{def: comp const}
\end{equation}
where $v_S \in\RR^d$ denotes the vector that agrees with $v$ in coordinates in $S$ and $0$ everywhere else and $v_{-S}\in\RR^d$ denotes $v - v_S$. 

\textbf{Notation.~}
Let $e_i$ be the $i$-th {canonical basis} vector.
We define $[x] = \{1,2,\ldots,x\}$.
Let $\supp(\theta)$ be the set of coordinate indices $i$ where $\theta_i \neq 0$.
We use $a \lesssim b $ to denote that there exists an absolute constant $c$ such that $a \le cb$.

\end{definition}

\section{Improved Linear Regression and Experimental Design for Sparse Models}
\label{sec:popart}

In this section, we discuss our novel sparse linear estimator \popart for the setting where the population covariance matrix is known and show its strong theoretical properties.
We then present a variation of \popart called \wpopart that amends a potential weakness of \popart, followed by our novel experimental design for \popart and discuss its merit over prior art.



\textbf{\popart (POPulation covariance regression with hARd Thresholding).}
Unlike typical estimators for the statistical learning setup, our main estimator \popart described in Algorithm~\ref{alg:popart} takes the population covariance matrix denoted by $Q$ as input.
We summarize our assumption for \popart.

\begin{assumption}\label{ass:popart}
(Assumptions on the input of \popart)
There exists $\mu$ such that the input data points $\{(\cova_t,\resp_t)\}_{t=1}^n$ satisfy that $\cova_t \sr{\text{i.i.d.}}{\sim} \mu$ and $Q=Q(\mu) := \EE_{X\sim \mu}[X X^\T]$.
Furthermore, $\resp_t = \inner{\theta^*}{\cova_t} + \eta_t$ with $\eta_t$ being zero-mean $\sig$-subgaussian noise.
Also, $R_0 \geq \max_{a \in \mathcal{A}} |\langle a, \theta^* - \theta_0 \rangle|$.
\end{assumption}



\begin{algorithm}[h]
\caption{ \popart (POPulation covariance regression with hARd Thresholding)}
\label{alg:popart}
\begin{algorithmic}[1]
\STATE \textbf{Input:} Samples $\{(\cova_t, \resp_t)\}_{t=1}^n$, the population covariance matrix $Q\in\RR^{d\times d}$, pilot estimator $\theta_0 \in \mathbb{R}^d$, an upper bound $R_0$ of $\max_{a \in \mathcal{A}} |\langle a, \theta^* - \theta_0 \rangle|$, failure rate $\delta$. 
\STATE \textbf{Output:} estimator $\hat{\theta}$
\FOR{$t=1,\ldots,n$ } 
\STATE $\tilde{\theta}_t = Q^{-1}\cova_t (\resp_t-\langle \cova_t, \theta_0\rangle )+ \theta_0$
\label{step:one-sample-estimator}
\ENDFOR
\STATE $\forall i\in[d], {\theta}_i'=\textsf{Catoni}(\{\tilde{\theta}_{ti}:=\langle\tilde{\theta}_{t}, e_i \rangle\}_{t=1}^n , \alpha_i, \frac{\delta}{2d})$  where $\alpha_i:= \sqrt{\frac{2\log \frac{2d}{\delta}}{n(R_0^2 + \sigma^2)(Q^{-1})_{ii}(1+ \frac{2\log \frac{2d}{\delta}}{n-2 \log \frac{2d}{\delta}})}}$
\label{step:catoni-i}
\STATE $\hat{\theta}\leftarrow \textsf{clip}_\lambda ({\theta'}):= [{\theta}'_i \one(|{\theta}'_i|>\lambda_i)]_{i=1}^d$ where $\lambda_i$ is defined in Proposition~\ref{prop:individual conf bound}.
\label{step:hard-threshold}
\RETURN $\hat{\theta}$
\end{algorithmic}
\end{algorithm}

\popart consists of several stages. In the first stage, for each $(\cova_t, \resp_t)$ pair, we create a one-sample estimator $\tilde{\theta}_t$ (step~\ref{step:one-sample-estimator}).
The estimator, $\tilde{\theta}_t$, can be viewed as a generalization of doubly-robust estimator~\cite{chernozhukov2019semi,dudik2011doubly} for linear models. 
Specifically, it is the sum of two parts: one is the pilot estimator $\theta_0$ that is a hyperparameter of \popart; the other is $Q(\mu)^{-1} \cova_t (\resp_t-\langle \cova_t, \theta_0 \rangle )$, an unbiased estimator of the difference $\theta^* - \theta_0$. 
Thus, it is not hard to see that $\tilde{\theta}_t$ is an unbiased estimator of $\theta^*$.
As we will see in Theorem~\ref{thm: main bounds of estimator}, the variance of $\tilde{\theta}_t$ is smaller when $\theta_0$ is closer to $\theta^*$, showing the advantage of allowing a pilot estimator $\theta_0$ as input.
If no good pilot estimator is available a priori, one can set $\theta_0=0$.

From the discussion above, it is natural to take an average of $\tilde{\theta}_t$. Indeed, when $n$ is large, the population covariance matrix $Q(\pi)$ is close to empirical covariance matrix $\hat{Q} := \frac1n \sum_{t=1}^n \cova_t \cova_t^\top$, which makes $\hat{\theta}_{\text{avg}} := \frac1n \sum_{t=1}^n \tilde{\theta}_t$ close to the ordinary least squares estimator $\hat{\theta}_{\text{OLS}} = \hat{Q}^{-1} (\frac1n\sum_{t=1}^n \cova_t \resp_t)$. 
However, for technical reasons, the concentration property of $\tilde{\theta}_{\text{avg}}$ is hard to establish.
This motivates \popart's second stage (step~\ref{step:catoni-i}), where, for each coordinate $i \in [d]$, we employ
Catoni's estimator \cite{lugosi2019mean} (see Appendix~\ref{sec:catoni} for a recap) to obtain an intermediate estimate for each $\theta_i^*$, namely $\theta_i'$. 


To use Catoni's estimator, we need to have an upper bound of the variance of $\theta_i'$ for its $\alpha_i$ parameter. A direct calculation yields that, for all $i\in[d]$ and $t\in[n]$,  $$\textrm{Var}(\tilde{\theta}_{ti})\leq \del{ \max_{a \in \cA} \inner{\theta^*-\theta_0}{a}^2  + \sigma^2 } \max_i (Q(\mu)^{-1})_{ii}$$ where $\tilde{\theta}_{ti}:=\langle \tilde{\theta}_t,e_i\rangle$. This implies that $\del{ R_0^2  + \sigma^2 } \max_i (Q(\mu)^{-1})_{ii}$ is an upper bound of $\textrm{Var}(\tilde{\theta}_{ti})$. By the standard concentration inequality of Catoni's estimator (see Lemma~\ref{lem:catoni-error}), we obtain the following estimation error guarantee for $\theta_i'$; the proof can be found in Appendix~\ref{appendix:proof-prop}.
Hereafter, all proofs are deferred to appendix unless noted otherwise.



\begin{proposition} \label{prop:individual conf bound} 
Suppose Assumption~\ref{ass:popart} holds.
In \popart, for $i\in [d]$, if $n\geq 2 \ln \frac{2d}{\delta}$, the following inequality holds with probability $1-\frac{\delta}{d}$:
$$|\theta_i' - \theta_i^*|< \sqrt{\frac{4(R_0^2 + \sigma^2) (Q(\mu)^{-1})_{ii}^2}{n} \log \frac{2d}{\delta}} =: \lambda_i$$
\end{proposition}



Proposition~\ref{prop:individual conf bound} shows that,
for each coordinate $i$, 
$(\theta_i' - \lambda_i, \theta_i' + \lambda_i)$ forms a confidence interval for $\theta_i^*$. 
Therefore, if $0 \notin (\theta_i' - \lambda_i, \theta_i' + \lambda_i)$, we can infer that $\theta_i^* \neq 0$, i.e., $i \in \supp(\theta^*)$. 
Based on the observation above, \popart's last stage (step~\ref{step:hard-threshold})
performs a hard-thresholding for each of the coordinates of $\theta'$, using the threshold $\lambda_i$ for coordinate $i$.
Thanks to the thresholding step, with high probability, $\hat{\theta}$'s support is contained in that of $\theta^*$, which means that all coordinates $i$ outside the support of $\th^*$ (typically the vast majority of the coordinates when $s \ll d$) satisfy $\hat\theta_i = \theta^*_i = 0$.
Meanwhile, for coordinate $i$'s in $\supp(\theta^*)$, the estimated value $\hat{\theta}_i$ is not too far from $\theta^*_i$.


The following theorem states \popart's estimation error bound in terms of its output  $\hat{\theta}$'s $\ell_\infty$, $\ell_0$, and $\ell_1$ errors, respectively. We remark that replacing hard thresholding in the last stage with soft thresholding enjoys similar guarantees.
%
\begin{theorem}\label{thm: main bounds of estimator}
Take Assumption~\ref{ass:popart}.
Let $H^2(Q):=\max_{i \in [d]} (Q^{-1})_{ii}$.
Then, \popart has the following guarantees
with probability at least $1-\delta$: 
\begin{enumerate}[label=(\roman*)]
    \item $\forall i \in [d], |\hat{\theta}_i - \theta_i^*|< 2 \sqrt{\frac{4 (R_0^2 + \sigma^2)(Q(\mu)^{-1})_{ii}}{n}\log \frac{2d}{\delta}}$ so $\|\hat{\theta}-\theta^*\|_\infty < 2 \sqrt{\frac{4 (R_0^2 + \sigma^2)H^2(Q(\mu))}{n}\log \frac{2d}{\delta}}, $ 
    \item $\textrm{supp}(\hat{\theta})\subset \textrm{supp}({\theta}^*)$ so $\|\hat{\theta}-\theta^*\|_0 \leq s$,
    \item $\|\hat{\theta}-\theta^*\|_1 \leq 2s \sqrt{\frac{4 (R_0^2 + \sigma^2)H^2(Q(\mu))}{n}\log \frac{2d}{\delta}}$
\end{enumerate}
\end{theorem}
Interestingly, \popart has no false positive for identifying the sparsity pattern and  enjoys an $\ell_\infty$ error bound, which is not available from Lasso, to our knowledge.
Unfortunately, a direct comparison with Lasso is nontrivial since the largest compatibility constant $\phi_0^2 (\hat{\Sigma}, s)$ is defined as the solution of the optimization problem~\eqref{def: comp const}, let alone the fact that $\phi_0^2 (\hat{\Sigma}, s)$ is a function of the empirical covariance matrix.
While we leave further investigation as future work, our experiment results in Section~\ref{sec:expr} suggest that there might be a case where \popart makes a meaningful improvement over Lasso.
\vspace{-0.25cm}
\begin{proof}[Proof of Theorem~\ref{thm: main bounds of estimator}]
{Let $\lambda := \max_i \lambda_i = \sqrt{\frac{4(R_0^2 + \sigma^2) H^2(Q(\mu))}{n} \log \frac{2d}{\delta}}$}
From Proposition \ref{prop:individual conf bound} and the union bound, one can check that \begin{equation} \label{eqn: result of prop1}
    \|{\theta'}-\theta^*\|_\infty < \lambda
\end{equation} with probability $1-\delta$. Therefore, the coordinates in $\textrm{supp}(\theta^*)^c$ will be thresholded out because of $\|{\theta'}-\theta^*\|_\infty \leq \lambda$. Therefore, (ii) holds and for all $i \in \textrm{supp}(\theta^*)^c$, $|\hat{\theta}_i - \theta_i^*|=0$. 

By definition, $\hat{\theta}=\textsf{clip}_\lambda ({\theta'})$, we can say that $\|\hat{\theta}-\theta'\|_\infty \leq \lambda$. Plus, by Eq. (\ref{eqn: result of prop1}), $\|\theta'-\theta^* \|_\infty \leq \lambda$. By the triangle inequality, $\|\theta^*-\hat{\theta}\|_\infty \leq 2 \lambda$. Therefore, (i) holds. 

Lastly, (iii) can be argued as follows:
\[
    \|\hat{\theta}-\theta^*\|_1 = \sum_{i \in [d]} |\hat{\theta}_i - \theta_i^*|
    \leq \sum_{i \in \textrm{supp}(\theta^*)^c} 0 + \sum_{i \in \textrm{supp}(\theta^*)} 2\lambda \leq 2 s \lambda.
    \qedhere
\]
\end{proof}
\vspace{-8pt}


\textbf{\wpopart: Improved guarantee by warmup.~}
One drawback of the \popart estimator is that its estimation error scales with $\sqrt{R_0^2 + \sigma^2 }$, which can be very large when $R_0$ is large.
One may attempt to use the fact that \popart allows a pilot estimator $\theta_0$ to address this issue since $R_0$ gets smaller as $\th_0$ is closer to $\th^*$.
However, it is a priori unclear how to obtain a $\theta_0$ close to $\theta^*$ as $\theta^*$ is the unknown parameter that we wanted to estimate in the first place. 

To get around this ``chicken and egg'' problem, we propose to introduce a warmup stage, which we call \wpopart (Algorithm \ref{alg:warm-popart}). 
\wpopart consists of two stages. 
For the first warmup stage, the algorithm runs $\popart$ with the zero vector as the pilot estimator and with the first half of the samples to obtain a coarse estimator denoted by $\hat\theta_0$  which guarantees that for large enough $n_0$, $\|\hat\theta_0 - \theta^*\|_1 \leq \sigma$. In the second stage, using $\hat\theta_0$ as the pilot estimator, it runs $\popart$ on the remaining half of the samples.



\begin{algorithm}[h]
\caption{\wpopart}
\begin{algorithmic}[1]

\STATE \textbf{Input:} Samples $\{(\cova_t, \resp_t)\}_{t=1}^{n_0}$, the population covariance matrix $Q\in\RR^{d\times d}$, an upper bound 
$R_{\max}$ of $\max_{a \in \cA} |\langle \theta^* , a \rangle| $, number of samples $n_0$, failure rate $\delta$.

\STATE \textbf{Output:} $\hat{\theta}$, an estimate of $\theta^*$
\STATE Run $\popart( \{( \cova_i, \resp_i )\}_{i=1}^{\lfloor n_0/2 \rfloor}, Q,  \vec{0}, \delta, R_{\max})$ to obtain ${\hat\theta_0}$, a coarse estimate of $\theta^*$ for the next step.
\label{step:coarse-estimation}

\STATE Run $\popart( \{( \cova_i, \resp_i )\}_{i= \lfloor n_0/2 \rfloor +1}^{n_0},Q, \hat\theta_0, \delta, \sigma)$ to obtain $\hat{\theta}$, an estimate of $\theta^*$.
\end{algorithmic}
\label{alg:warm-popart}
\end{algorithm}

The following corollary states the estimation error bound of the output estimator $\hat{\theta}$. Compared with \popart's $\ell_1$ recovery guarantee, \wpopart's $\ell_1$ recovery guarantee (Equation~\eqref{eqn:warm-popart-l1}) has no dependence on $R_{\max}$; its dependence on $R_{\max}$ only appears in the lower bound requirement for $n_0$. 



\begin{corollary}\label{cor:warm-popart}
Take Assumption~\ref{ass:popart} without the condition on $R_0$.
Assume that $R_{\max} \geq \max_{a \in \mathcal{A}} |\langle a, \theta^* \rangle|$, and $n_0>\frac{32s^2(R_{\max}^2 + \sigma^2)H^2 (Q(\mu))}{\sigma^2} \log \frac{2d}{\delta}$.
Then, \wpopart has, with probability at least $1-2\delta$,
\begin{equation}
    \|\hat{\theta}-\theta^*\|_1 \leq 8 s \sigma  \sqrt{\frac{H^2(Q(\mu))\ln \frac{2d}{\delta}}{n_0}}.
    \label{eqn:warm-popart-l1}
\end{equation}
\end{corollary}
\begin{remarks}
In Algorithm~\ref{alg:warm-popart}, we choose $\popart$ as our coarse estimator, but we can freely change the coarse estimation step (step~\ref{step:coarse-estimation}) to other principled estimation methods (such as Lasso) without affecting the main estimation error bound~\eqref{eqn:warm-popart-l1}; the only change will be the lower bound requirement of $n_0$ to another problem-dependent constant.
\end{remarks}

\begin{remarks} \wpopart requires the knowledge of $R_{\max}$, an upper bound of $\max_{a \in \cA} |\langle \theta^* , a \rangle|$; this requirement can be relaxed by changing the last argument of the coarse estimation step (step~\ref{step:coarse-estimation}) from $R_{\max}$, to some function $f(n_0)$ such that $f(n_0) = \omega(1)$ and $f(n_0) = o(\sqrt{n_0})$ (say, $\sigma n_0^{\frac14}$);
with this change, a result analogous to Corollary~\ref{cor:warm-popart} can be proved with a different lower bound requirement of $n_0$.
\end{remarks}

\textbf{A novel and efficient experimental design for sparse linear estimation.~}
%
In the experimental design setting where the learner has freedom to design the underlying sampling distribution $\mu$, the $\ell_1$ error bound of \popart and \wpopart naturally motivates a design criterion.
Specifically, we can choose $\mu$ that minimizes $H^2(Q(\mu))$, which gives the lowest estimation error guarantee. 
We denote the optimal value of $H^2(Q(\mu))$ by
\begin{gather}
    H_*^2 := \underset{\mu \in \mathcal{P}(\cA)}{\textrm{min}}\max_{i\in[d]} (Q(\mu)^{-1})_{ii} ~. \label{def: H2}
\end{gather}

The minimization of $H^2(Q(\mu))$ is a convex optimization problem, which admits efficient methods for finding the solution. 
Intuitively, $H_*^2$ captures the geometry of the action set $\cA$.



To compare with previous studies that design a sampling distribution for Lasso, we first review the standard $\ell_1$ error bound of Lasso.
\begin{theorem}
\label{thm:lasso}
(\citet[Theorem 6.1]{bv11}) With probability at least $1-2\delta$, the $\ell_1$-estimation error of the optimal Lasso solution $\hat{\theta}_{\Lasso}$ \cite[Eq. (2.2)]{bv11} with $\lambda = \sqrt{2\log(2d/\delta)/n}$ satisfies 
$$ \| \hat{\theta}_{\Lasso}-\theta^*\|_1 \leq \frac{s \sigma}{\phi_0^2 (\hat{\Sigma}, s)}\sqrt{\frac{2 \log (2d/\delta) }{n}},$$

where $\phi_0 (\hat{\Sigma}, s)^2$ is the compatibility constant with respect to the empirical covariance matrix $\hat{\Sigma} = \frac{1}{n} \sum_{t=1}^n X_t X_t^\top$  and the sparsity $s$ in Eq. \eqref{def: comp const}.
\end{theorem}

Ideally, for Lasso, experiment design which minimizes the compatibility constant will guarantee the best estimation error bound within a fixed number of samples $n$. However, naively, the computation of the compatibility constant is intractable since Eq.~\eqref{def: comp const} is a combinatorial optimization problem which is usually difficult to compute.
One simple approach taken by~\citet{hao2020high} is to use the following computationally tractable surrogate of $\phi_0^2 (\hat{\Sigma}, s)$:
\begin{gather}
    \mathcal{C}_{\min}  (\cA):= \underset{\mu \in \mathcal{P}(\cA)}{\textrm{max}}\lambda_{\tmin} (Q(\mu)) \label{def: Cmin}
\end{gather}
\ja{we will use $\cC_{\min}$ instead of $\cC_{\min}(\cA)$ when it is clear from the context. } where $\lambda_{\tmin}(A)$ denotes the minimum eigenvalue of a matrix $A$.
With the choice of sampling distribution $\mu = \underset{\mu \in \mathcal{P}(\cA)}{\textrm{argmax}}\lambda_{\tmin} (Q(\mu))$, and $n \geq \tilde{\Omega}(\frac{s \cdot \polylog(d)}{{\Cmin}^2})$, with high probability, $\phi_0^2 (\hat{\Sigma},s) \geq \Cmin/2$ holds~\cite[][Theorem 1.8]{rudelson2012reconstruction},
and one can replace $\phi_0 (\hat{\Sigma}, s)$ to ${\Cmin}/2$ in Theorem \ref{thm:lasso} to get the following corollary:
\begin{corollary}
\label{cor:lasso with Cmin} With probability at least $1-\exp(-c n)-2\delta$ for some universal constant $c$, the $\ell_1$-estimation error of the optimal Lasso solution $\hat{\theta}_{\Lasso}$ satisfies
\begin{equation}\label{eqn:lasso with Cmin}
    \| \hat{\theta}_{\Lasso}-\theta^*\|_1 \leq \frac{2s \sigma}{{\Cmin}}\sqrt{\frac{2 \log (2d/\delta) }{n}},
\end{equation} 
\end{corollary}



The following proposition shows that our estimator has a better error bound compared to the surrogate experimental design for Lasso of~\citet{hao2020high}.

\begin{proposition}\label{prop:H2 vs Cmin}
We have $ H_*^2 \leq \mathcal{C}_{\min}^{-1} \leq d H_*^2$. Furthermore, there exist arm sets for which either of the inequalities is tight up to a constant factor.
\end{proposition}

Therefore, our new estimator has $\ell_1$ error guarantees at least a factor $\mathcal{C}_{\min}^{-1/2}$ better
than that provided by~\cite{hao2020high}, as follows: when we choose the $\mu$ as the solution of the Eq. \eqref{def: H2}, then
$$ (\text{RHS of \eqref{eqn:warm-popart-l1}}) \lesssim s \sigma H_* \sqrt{\frac{\ln(2d/\delta)}{n}} \lesssim s \sigma {\Cmin}^{-1/2} \sqrt{\frac{\ln(2d/\delta)}{n}} \lesssim s\sigma {\Cmin}^{-1} \sqrt{\frac{\ln(2d/\delta)}{n}} \lesssim (\text{RHS of \eqref{eqn:lasso with Cmin}})$$

In addition, we also prove that there exists a case where our estimator has an $d/s$-order better error bound compared to the traditional lasso bound in Theorem~\ref{thm:lasso}, although this is not in terms of the compatibility constant of the empirical covariance matrix $\hat{\Sigma}$.

\begin{proposition}\label{prop:H2 vs compat} There exists an action set $\mathcal{A}$ and an absolute constant $C_1>0$ such that $$H_* <C_1 \frac{s}{d} \times \frac{1}{\phi_0^2 (\Sigma, s)}$$
\end{proposition}
For the detailed proof about Proposition \ref{prop:H2 vs Cmin} and Proposition \ref{prop:H2 vs compat}, see Section \ref{appendix: example of H2 and Cmin} in Appendix. 


\section{Improved Sparse Linear Bandits using \wpopart}
\label{sec:bandits}


We now apply our new \wpopart sparse estimation algorithm to design new sparse linear bandit algorithms.
Following prior work~\cite{hao2020high}, we adopt the  classical Explore-then-Commit (ETC) framework for algorithm design, and use \popart with experimental design to perform exploration. 
As we will see, the tighter $\ell_1$ estimation error bound of our \popart-based estimators helps us obtain an improved regret bound. 

\begin{algorithm}[h]
\caption{Explore then commit with \wpopart}
\begin{algorithmic}[1]
\STATE Input: time horizon $n$, action set $\cA$, warm-up exploration length $n_0$, failure rate $\delta$, reward threshold parameter $R_{\max}$, an upper bound of $\max_{a \in \cA} |\langle \theta^* , a \rangle| $.
\STATE Solve the optimization problem in Eq.~\eqref{def: H2} and denote the solution as $\mu_*$
\FOR{$t=1, \ldots, n_0$}
\STATE Independently pull the arm $A_t$ according to $\mu_*$ and receives the reward $r_t$
\ENDFOR 
\STATE Run $\wpopart(\{A_t\}_{t=1}^{n_0}, \{r_t\}_{t=1}^{n_0}, Q(\mu_*), \delta, R_{\max})$ to obtain $\hat{\theta}$, an estimate of $\theta^*$.
\FOR{$t=n_0 +1,\ldots,n$}
\STATE Take action $A_t = \argmax_{a \in \cA} \inner{\hat{\theta}}{a}$, receive reward $r_t = \inner{\theta^*}{A_t} + \eta_t$
\ENDFOR
\end{algorithmic}
\label{alg:etc-sparse}
\end{algorithm}

\textbf{Sparse linear bandit with \wpopart.~}
%
Our first new algorithm, Explore then Commit with \wpopart (Algorithm~\ref{alg:etc-sparse}), proceeds as follows. For the exploration stage, which consists of the first $n_0$ rounds, it solves the optimization problem~\eqref{def: H2} to find $\mu_*$, the optimal sampling distribution for $\popart$ and samples from it to collect a dataset for the estimation of $\theta^*$. Then, we use this dataset to compute the \wpopart estimator $\hat{\theta}$. Finally, in the commit stage, which consists of the remaining $n-n_0$ rounds, we take the greedy action with respect to $\hat{\theta}$. We prove the following regret guarantee of Algorithm~\ref{alg:etc-sparse}:


\begin{theorem}
\label{thm:etc-sparse}
If Algorithm~\ref{alg:etc-sparse} has input time horizon $n>16\sqrt{2}\frac{R_{\max} (R_{\max}^2 + \sigma^2)^{3/2} H_*^2 s^2}{\sigma^4} \log \frac{2d}{\delta}$, action set $\cA \subset [-1,+1]^d$, and exploration length $n_0 = 4(s^2 \sigma^2 H_*^2 n^2 \log \frac{2d}{\delta}R_{\max}^{-2})^{\frac{1}{3}}$, $\lambda_1 = 4\sigma \sqrt{\frac{H_*^2}{n_0}\log \frac{2d}{\delta}}$, then with probability at least $1-2\delta$,
$
\Reg(n) \leq 8R_{\max}^{1/3}(s^2 \sigma^2 H_*^2 n^2 \log \frac{2d}{\delta})^{\frac{1}{3}}
$.
\end{theorem}

\begin{proof}
From Corollary \ref{cor:warm-popart}, $ \|\hat{\theta} - \theta^*\|_1 \leq 2s \lambda_1$ with probability at least $1-2\delta$. Therefore, with probability $1-2\delta$, 
\begin{align*}
    \textrm{Reg}(n) &\leq  R_{\max} n_0 + (n-n_0)\|\hat{\theta} - \theta^*\|_1 \leq R_{\max} n_0 + 2sn\lambda_1 = R_{\max} n_0 + 8sn\sigma \sqrt{\frac{H_*^2}{n_0}\log \frac{2d}{\delta}}
\end{align*}
and optimizing the right hand side with respect to $n_0$ leads to the desired upper bound.  
\end{proof}


Compared with~\citet{hao2020high}'s regret bound $\tilde{O}((R_{\max} s^2 \sigma^2 {\Cmin}^{-2} n^2)^{1/3})$\footnote{This is implicit in~\cite{hao2020high} -- they assume that $\sigma=1$ and do not keep track of the dependence on $\sigma$.}
, Algorithm~\ref{alg:etc-sparse}'s regret bound $\tilde{O}((R_{\max} s^2 \sigma^2 H_*^2 n^2)^{1/3})$
is at most $\tilde{O}((R_{\max} s^2 \sigma^2 {\Cmin}^{-1} n^2)^{1/3})$, which is at least a factor ${\Cmin}^{\frac13}$ smaller. As we will see in Section~\ref{sec:lower-bound}, we show that the regret upper bound provided by Theorem~\ref{thm:etc-sparse} is unimprovable in general, answering an open question of~\cite{hao2020high}.



%
\textbf{Improved upper bound with minimum signal condition.~} Our second new algorithm, Algorithm~\ref{alg:phase-elim}, similarly uses \wpopart under an additional minimum signal condition.

\begin{assumption}[Minimum signal]There exists a known lower bound $m > 0$ such
that $\min_{j\in \textrm{supp}(\theta^*)} |\theta_j^*| > m$.
\end{assumption}

At a high level, Algorithm~\ref{alg:phase-elim} uses the first $n_2$ rounds for identifying the support of $\theta^*$; the $\ell_\infty$ recovery guarantee of $\wpopart$ makes it suitable for this task. Under the minimal signal condition and a large enough $n_2$, it is guaranteed that $\hat{\theta}_2$'s support equals exactly the support of $\theta^*$. After identifying the support of $\theta^*$, Algorithm~\ref{alg:phase-elim} treats this as a $s$-dimensional linear bandit problem by discarding the remaining $d-s$ coordinates of the arm covariates, 
and perform phase elimination algorithm \citep[Section 22.1]{lattimore18bandit} therein. The following theorem provides a regret upper bound of Algorithm \ref{alg:phase-elim}. 

\begin{algorithm}[h]
\caption{Restricted phase elimination with \wpopart}
\begin{algorithmic}[1]
\STATE Input: time horizon $n$, finite action set $\cA$, minimum signal $m$, failure rate $\delta$, reward threshold parameter $R_{\max}$, an upper bound of $\max_{a \in \cA} |\langle \theta^* , a \rangle|$
\STATE Solve the optimization problem in Eq. \ref{def: H2} and denote the solutions as $Q$ and $\mu_*$, respectively.
\STATE Let $n_2 = \max(\frac{256\sigma^2 H_*^2}{m^2} \log \frac{2d}{\delta} , \frac{32s^2(R_{\max}^2 + \sigma^2)H_*^2}{\sigma^2} \log \frac{2d}{\delta})$
\FOR{$t=1,\ldots,n_2$}
    \STATE Independently pull the arm $A_t$ according to $\mu_*$ and receives the reward $r_t$
\ENDFOR 
\STATE $\hat{\theta}_2 = \wpopart(\{ A_t\}_{t=1}^n, \{ R_t\}_{t=1}^n, Q, \delta,R_{\max})$
\STATE Identify the support $\hS = \mathrm{supp}(\hat{\theta}_2)$
\FOR{$t=n_2+1,\ldots,n$}
\STATE Invoke phased elimination algorithm for linear bandits on $\hS$
\ENDFOR
\end{algorithmic}
\label{alg:phase-elim}
\end{algorithm}

\begin{theorem} \label{thm:with minimum signal}
If Algorithm \ref{alg:phase-elim} has input time horizon $n>\max(\frac{2^8\sigma^2 H_*^2}{m^2} , \frac{2^5 s^2(R_{\max}^2 + \sigma^2)H_*^2}{\sigma^2} )\log \frac{2d}{\delta}$, action set $\cA \subset [-1,1]^d$, upper bound of the reward $R_{\max}$, then with probability at least $1-2\delta$, the following regret upper bound of the Algorithm \ref{alg:phase-elim} holds: for universal constant $C>0$, 
$$ \textrm{Reg} (n) \leq \max(\frac{2^8\sigma^2 H_*^2}{m^2} \log \frac{2d}{\delta} , \frac{2^5s^2 (R_{\max}^2 + \sigma^2)H_*^2}{\sigma^2} \log \frac{2d}{\delta}) + C\sigma\sqrt{sn \log (|\cA|n)}$$

\end{theorem}

For sufficiently large $n$, the second term dominates, and we obtain an $O(\sqrt{sn})$ regret upper bound. Theorem~\ref{thm:with minimum signal} provides two major improvements compared to~\citet[][Algorithm 2]{hao2020high}. First,
when $m$ is moderately small (so that the first subterm in the first term dominates), 
it shortens the length of the exploration phase $n_2$ by a factor of $s \cdot \frac{{\Cmin}}{H_*^2}$. Second, compared with the regret bound  
$\tilde{O}( \sqrt{\frac{9\lambda_{\max}(\sum_{i=1}^{n_2}A_i A_i^\top/n_2)}{\mathcal{C}_{\min}}} \sqrt{sn} )$ provided by~\cite{hao2020high},
our main regret term $\tilde{O}(\sqrt{sn})$ is more interpretable and can be much lower. 


\vspace{-6pt}
\section{Matching lower bound}
\label{sec:lower-bound}
\vspace{-6pt}
We show the following theorem that establishes the optimality of Algorithm~\ref{alg:etc-sparse}. This solves the open problem of \citet[][Remark 4.5]{hao2020high} on the optimal order of regret in terms of sparsity and action set geometry in sparse linear bandits. 
\begin{theorem}\label{thm:lower}
For any algorithm, any $s, d, \kappa$ that satisfies
$s>2000, \kappa \in (0,1), n>8\kappa s^2$ and $d \geq \max (n^{1/3} s^{4/3} \kappa^{-4/3},(s+1)^2)$, there exists a linear bandit environment an action set $\cA$ and a $s$-sparse $\theta \in \RR^d$, such that $\mathcal{C}_{\min}^{-1} \leq \kappa^{-2}$, $R_{\max} \leq 2$, $\sigma = 1$, and 
\[
\textrm{Reg}_n \geq \Omega( \kappa^{-2/3} s^{2/3} n^{2/3})~.
\]
\end{theorem}

We give an overview of our lower bound proof techniques, and defer the details to Appendix \ref{sec:proof-lb}.

\paragraph{Change of measure technique.~}
Generally, researchers prove the lower bound by comparing two instances based on the information theory inequalities, such as  Pinsker's inequality, or Bregtanolle-Huber inequality. In this proof, we also use two instances $\theta$ and $\theta'$, but we use the change of measure technique, to help lower bound the probability of events more freely. Specifically, for any event $A$,
\begin{align}\label{eqn:change-of-measure}
    \PP_\th (A) 
    = \EE_{\th} [\one_A]
    = 
    \EE_{\th'}\sbr{ \one_A \prod_{t=1}^n\frac{p_\th(r_t |a_t)}{p_{\theta'}(r_t |a_t)} } 
    \gtrsim \EE_{\theta'} \sbr{ \one_A \exp\del{-\sum_{t=1}^n \langle A_t , \theta-\theta' \rangle^2} }~.
\end{align}

\paragraph{Symmetrization.~}
We utilize the algorithmic  symmetrization technique of~\citet{simchowitz2017simulator, bubeck11pure-tcs}, which makes it suffice to focus on proving lower bounds against  symmetric algorithms.

\begin{definition}[Symmetric Algorithm]
An algorithm $\textsf{Alg}$ is \emph{symmetric} if for any permutation $\pi \in \textit{Sym}(d)$, $\theta \in \mathbb{R}^{d}$, $\{a_t\}_{t=1}^n \in \mathcal{A}^n$,
$$ \PP_{\theta, \textsf{Alg}} (A_1 = a_1, \cdots , A_n = a_n) = \PP_{{\pi} (\theta) , \textsf{Alg}}(A_1 = \pi(a_1 ), \cdots , A_n = \pi(a_n ))$$
where for vector $v$, $\pi(v) \in \RR^d$ denotes its permuted version that moves $v_i$ to the $\pi(i)$-th position.
\end{definition}
This approach can help us to exploit the symmetry of $\theta'$ to lower bound the right hand side of~\eqref{eqn:change-of-measure}; below, $\Pi := \cbr{\pi': \pi(\theta') = \theta'}$ is the set of permutations that keep $\theta'$ invariant, {and $A$ is an event invariant under $\Pi$}:
\begin{align*}
    \text{~\eqref{eqn:change-of-measure}}
    \geq
    \frac{1}{|\Pi|} \sum_{\pi \in \Pi} 
    \EE_{ \theta'} \sbr{ \one_A \exp(-\sum_{t=1}^n \langle \pi^{-1}(A_t) , \theta-\theta' \rangle^2) }    \geq  
    \EE_{ \theta'} \sbr{ \one_A \exp\del{ -\sum_{t=1}^n \frac{1}{|\Pi|} \sum_{\pi \in \Pi} \langle \pi^{-1}(A_t) , \theta-\theta' \rangle^2} }
\end{align*}
which helps us use combinatorial tools over the actions for the lower bound proof.

\vspace{-6pt}
\section{Experimental results}
\label{sec:expr}
\vspace{-6pt}

We evaluate the empirical performance of \popart and our proposed experimental design, along with its impact on sparse linear bandits. One can check our code from here: \url{https://github.com/jajajang/sparse}. 

\begin{figure}[h]        
    \centering
    \begin{tabular}{cc}
           \toprule
            Case 1 & Case 2 \\
            \midrule
                 \begin{tabular}{l}
                 \includegraphics[width=0.4\linewidth]{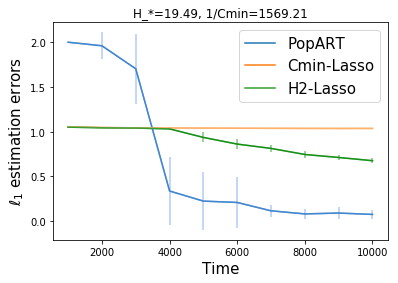}
                 \end{tabular}&
                 \begin{tabular}{l}
                 \includegraphics[width=0.4\linewidth]{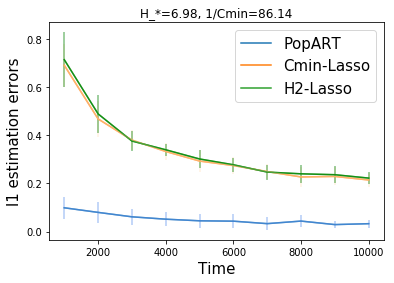}
                 \end{tabular}
    \\
    \begin{tabular}{l}
    \includegraphics[width=0.4\linewidth]{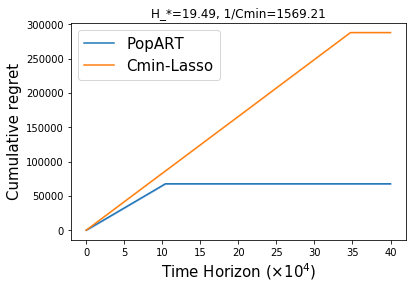}\end{tabular}&\begin{tabular}{l} \includegraphics[width=0.4\linewidth]{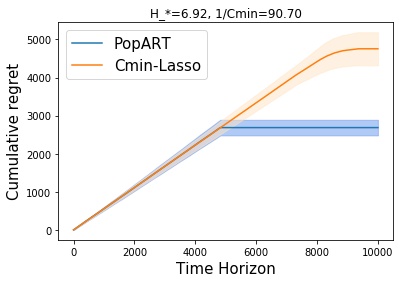}
    \end{tabular}
   \\\bottomrule
        \end{tabular}
    \caption{Experiment results on $\ell_1$ estimation error cumulative regret. 
    }
    \label{fig:table_of_figures}
\end{figure}

For sparse linear regression and experimental design, we compare our algorithm \popart with $\mu$ being the solution of~\eqref{def: H2} with two baselines.
The first baseline denoted by $C_{\min}$-Lasso is the method proposed by~\citet{hao2020high} that uses Lasso with sampling distribution $\mu$ defined by~\eqref{def: Cmin}.
The second baseline is $H^2$-Lasso, uses Lasso with sampling distribution $\mu$ defined by~\eqref{def: H2}, which is meant to observe if Lasso can perform better with our experimental design and to see how \popart is compared with Lasso as an estimator since they are given the same data. 
Of course, this experimental design is favored towards \popart as we have optimized the design for it, so our intention is to observe if there ever exists a case where \popart works better than Lasso.



For sparse linear bandits, we run a variant of our Algorithm~\ref{alg:etc-sparse} that uses \wpopart in place of \popart for simplicity.
As a baseline, we use ESTC~\cite{hao2020high}.
For both methods, we use the exploration length prescribed by theory.
We consider two cases:
\begin{itemize}
    \item \textbf{Case 1: Hard instance where $H_*^2 \ll \mathcal{C}_{\min}^{-1}$.~} We use the action set constructed in Appendix \ref{example:worst case of Cmin and H2} where $H_*^2$ and $\mathcal{C}_{\min}$ shows a gap of $\Theta(d)$. We choose $d=10$, $s=2$, $\sigma=0.1$.
    \item \textbf{Case 2. General unit vectors.~} In this case, we choose $d=30$, $s=2$, $\sigma=0.1$ and the action set $\mathcal{A}$ consists of $|\mathcal{A}|=3d=90$ uniformly random vectors on the unit sphere. 
\end{itemize}

We run each method 30 times and report the average and standard deviation of the $\ell_1$ estimation error and the cumulative regret in Figure~\ref{fig:table_of_figures}.


\paragraph{Observation.} As we expected from the theoretical analysis, our estimator and bandit algorithm outperform the baselines. 
In terms of the $\ell_1$ error, for both cases, we see that \popart converges much faster than ${\Cmin}$-Lasso for large enough $n$.
Interestingly, $H^2$-Lasso also improves by just using the design computed for \popart in case 1.
At the same time, $H^2$-Lasso is inferior than \popart even if they are given the same data points.
While the design was optimized for \popart and \popart has the benefit of using the population covariance, which is unfair, it is still interesting to observe a significant gap between \popart and Lasso.
For sparse linear bandit experiments, while ESTC requires exploration time almost the total length of the time horizon, ours requires a significantly shorter exploration phase in both cases and thus suffers much lower regret.



\vspace{-8pt}
\section{Conclusion}
\label{sec:conclusion}
\vspace{-8pt}

We have proposed a novel estimator \popart and experimental design for high-dimensional linear regression. 
\popart has not only enabled accurate estimation with computational efficiency but also led to improved sparse linear bandit algorithms.
Furthermore, we have closed the gap between the lower and upper regret bound on an important family of instances in the data-poor regime.

Our work opens up numerous future directions.
For \popart, we speculate that $(Q(\mu)^{-1})_{ii}$ is the statistical limit for testing whether $\theta^*_i = 0$ or not -- it would be a valuable investigation to prove or disprove this.
We believe this will also help investigate whether the dependence on $H_*^2$ in our regret upper bound is unimprovable (note our matching lower bound is only for a particular family of instances).
Furthermore, it would be interesting to investigate whether we can use \popart without relying on the population covariance; e.g., use estimated covariance from an extra set of unlabeled data or find ways to use the empirical covariance directly.
For sparse linear bandits, it would be interesting to develop an algorithm that achieves the data-poor regime optimal regret and data-rich regime optimal regret $\sqrt{sdn}$ simultaneously.
Furthermore, it would be interesting to extend our result to changing arm set, which poses a great challenge in planning.

\begin{ack}
We thank Ning Hao for helpful discussions on theoretical guarantees of Lasso.
Kwang-Sung Jun is supported by Data Science Academy and Research Innovation \& Impact at University of Arizona. 
\end{ack}

\bibliographystyle{abbrvnat}
\bibliography{library-shared}

\appendix


\clearpage
\part{Appendix} 


\parttoc 
  




\section{Related work}
\label{sec:related}

\paragraph{Sparse linear bandits.} The sparse linear bandit problem is a natural extension of sparse linear regression to the bandit setup where the goal is to enjoy low regret in the high-dimensional setting by levering the sparsity of the unknown parameter $\th^*$. 
The first study we are aware of is \citet{ay12online} that achieves a $\tilde O(\sqrt{sdn})$ regret bound with a computationally intractable method, which is later shown to be optimal by~\citet[Section 24]{lattimore18bandit} yet is not computationally efficient.
Since then, several approaches have been proposed.
A large body of literature either assumes that the arm set is restricted to a continuous set (e.g., a norm ball) ~\cite{carpentier2012bandit,lattimore2015linear}
or that the set of available arms at every round is drawn in a time-varying manner, and playing arms greedily still induces a `nice' arm distribution such as satisfying compatibility or restricted eigenvalue conditions~\cite{bastani2020online,kim19doubly,sivakumar2020structured,oh2021sparsity}.
These assumptions allow them to leverage existing theoretical guarantees of Lasso.
In contrast, we follow~\citet{hao2020high} and consider arm sets that are fixed throughout the bandit game without making further assumptions about the arm set.  
While this setup is interesting in its own for not having restrictive assumptions, it is also an important stepping stone towards efficient bandit algorithms for the more generic yet challenging setup of changing arm sets without any distributional assumptions.
Our work is a direct improvement over~\citet{hao2020high}, in that we close the gap between upper and lower bounds on the optimal worst-case regret; we refer to Table~\ref{table: results} for a detailed comparison.




\paragraph{Sparse linear regression.} 
Natural attempts for solving sparse linear bandits are to turn to existing results from sparse linear regression.
While best subset selection (BSS) is a straightforward approach of trying all the possible sparsity patterns that achieves good guarantees, its computational complexity is prohibitive~\cite{foster94risk}. 
As a computationally efficient alternative, Lasso is arguably the most popular approach for sparse linear regression for its simplicity and effectiveness~\cite{tibshirani96regression}. 
However, Lasso has an inferior $\ell_1$ norm error bound than BSS, perhaps due to its bias~\cite{vandegeer18ontight}. 
Rather than turning to existing results from sparse linear regression, we propose a novel estimator, \popart, by leveraging the fact that the setup allows us to design the sampling distribution, which allows a better $\ell_1$ norm error bound than Lasso except for the dependence on the range of the mean response variable.


\paragraph{Experimental design.}
In the linear bandit field, researchers often use experimental design to get the best estimator within the limited budget \cite{soare2014best,tao2018best,camilleri2021high,fiez2019sequential, mason2021nearly}. Especially, there were a few attempts using the population covariance based estimator instead of the traditional empirical covariance matrix \cite{mason2021nearly, tao2018best}. 
However, our study is the first approach that designs the experiment for minimizing the variance of each coordinate of the estimator uniformly, to the best of our knowledge.

For experimental design for sparse linear regression, \citet{ravi2016experimental} propose heuristic approaches that ensures the design distribution satisfy incoherence conditions and restricted isometry property (RIP). \citet{eftekhari2020design} study the design of $c$-optimal experiments in sparse regression models, where the goal is to estimate $\inner{c}{\theta^*}$ with low error for some $c \in \RR^d$; our experimental design task can be seen as simultaneously estimating $\inner{c}{\theta^*}$ for all $c = e_1, \ldots, e_d$. \citet{huang2020optimal} propose algorithms for optimal experimental design, tailored to minimizing the asymptotic variance of the debiased Lasso estimator~\cite{javanmard2014confidence}. 
In contrast, our results are based on finite-sample analyses.

In the theoretical computer science literature, a line of work on sketching also provides provably efficient compressed sensing and sparse recovery algorithms~\cite[See][for an overview]{gilbert2010sparse}; however, they mostly focus on using measurements (covariates) that are in $\cbr[0]{0,1}^d$  and $\cbr[0]{-1,1}^d$, as opposed to general measurement sets in $\RR^d$. 

\paragraph{Regression with the population covariance matrix.}
There are a few studies that consider regression with the population covariance matrix:
\citet{camilleri2021high} devise the novel scheme for the experimental design for the kernel bandits and obtain a new estimator called RIPS that leverages the population covariance matrix and robust mean estimator like \popart. 
\citet{mason2021nearly} solve the level set estimation problem using RIPS.
Tao et al. \cite{tao2018best} also employ a similar estimator, but they do not use robust mean estimators and result in a weaker form of error bound involving additional lower order terms.
The main difference of our work from all these papers is that they do not address sparse linear models. 
In particular, they do not perform thresholding nor  provide $l_\infty$ or $l_1$ recovery guarantees for the sparse parameter.

\section{Catoni's Estimator}
\label{sec:catoni}

\begin{definition}[Catoni's estimator~\cite{catoni2012challenging}] \label{def: Catoni}
For the i.i.d random variables $Z_1, \cdots, Z_n$, Catoni's mean estimator $\textrm{Catoni}(\{Z_{i}\}_{i=1}^n, \delta, \alpha)$ with error rate $\delta$ and the weight parameter $\alpha$ is defined as the unique value $y$ which satisfies 
$$ \sum_{i=1}^n \psi(\alpha(Z_n - y))=0$$
where $\psi(x):= \textrm{sign}(x)\log (1+|x|+x^2/2)$.
\end{definition}

\begin{lemma}[Catoni's estimator guarantee~\cite{catoni2012challenging}]
\label{lem:catoni-error}
For the i.i.d random variable $X_1, \cdots, X_n$ with mean $\mu$, let $\hat{\mu}$ be their Catoni's estimator with error rate $\delta$ with the weight parameter  $\alpha:= \sqrt{\frac{2\log \frac{1}{\delta}}{n\textrm{Var}(X_1)(1+ \frac{2\log \frac{1}{\delta}}{n-2 \log \frac{1}{\delta}})}} $.
Then with probability at least $1- 2\delta$, the following inequality holds:
$$ |\hat{\mu}-\mu|< \sqrt{\frac{2\textrm{Var}(X_1) \log \frac{1}{\delta} }{n-\log \frac{1}{\delta}}}
$$
\end{lemma}

\section{Proofs for \popart and \wpopart}
\label{sec:proof-popart}
\subsection{Proof for Proposition \ref{prop:individual conf bound}}\label{appendix:proof-prop}
\begin{proof}
To lighten the notation, in this proof, we write $Q= Q(\mu)$, and let $(\tilde{\theta}, A, \eta)$ denote random vectors distributed identically to $(\tilde{\theta}_1, A_1, \eta_1)$. 
First, observe that $\EE\sbr{\tilde{\theta}}  = \theta^*$. 
We now use the law of total variance to decompose the covariance matrix of  $\tilde{\theta}$, by first conditioning on $\cova$:
\begin{align*}
  \EE \sbr{ (\tilde{\th}-\th^*)(\tilde{\th}-\th^*)^\T }
  = &
  \EE \sbr{ (\EE[\tilde{\th} \mid \cova] - \th^*)(\EE[\tilde{\th} \mid \cova ] - \th^*)^\T }
  +
  \EE \sbr{ (\tilde{\th} -\EE[\tilde{\th} \mid \cova])(\tilde{\th}-\EE[\tilde{\th} \mid \cova ])^\T }
\end{align*}


For the first term, 
\begin{align*}
\EE \sbr{ (\EE[\tilde{\th} \mid \cova] - \th^*)(\EE[\tilde{\th} \mid \cova ] - \th^*)^\T } 
&= 
\EE \sbr{ (\EE[\tilde{\th} \mid \cova]-\theta_0)(\EE[\tilde{\th} \mid \cova ]-\theta_0)^\T } - (\th^*-\th_0) (\th^*-\th_0)^\T \\
\preceq &
\EE \sbr{ (\EE[\tilde{\th} \mid \cova]-\theta_0)(\EE[\tilde{\th} \mid \cova ]-\theta_0)^\T } 
\\= & \EE \sbr{ Q^{-1} \cova (\cova^\T (\theta^*-\theta_0))^2 \cova^\T Q^{-1} }\\
\preceq & R_0^2 \EE \sbr{ Q^{-1} \cova \cova^\T Q^{-1} } = R_0^2 Q^{-1}
\end{align*}

For the second term, 
\begin{align*}
\EE \sbr{ (\tilde{\th}-\EE[\tilde{\th} \mid \cova])(\tilde{\th}-\EE[\tilde{\th} \mid \cova ])^\T }
= \EE[ Q^{-1} \cova \cova^\T Q^{-1} \eta^2 ]
= \sig^2 Q^{-1}.
\end{align*}

Combining the above two bounds, we have $\var(\tilde{\theta}) =  \EE \sbr{ (\tilde{\th}-\th^*)(\tilde{\th}-\th^*)^\T } \preceq (R_0^2 + \sigma^2) Q^{-1}$. Therefore, we can bound $\textrm{Var}(\tilde{\theta}_i)$ as follows:
\begin{align*}
    \textrm{Var}(\tilde{\theta}_i) &= \EE[(e_i^\top (\tilde{\theta} - \theta^*))^2] \leq (R_0^2 + \sigma^2) (Q^{-1})_{ii}
\end{align*} 
By the theoretical guarantee of the Catoni's estimator (Lemma~\ref{lem:catoni-error} in the Appendix), the desired inequality holds. 
\end{proof}

\subsection{Full version of Corollary~\ref{cor:warm-popart} and its proof}

\begin{corollary}\label{cor:warm-popart-fullversion}
If \wpopart receives inputs $\cbr{\cova_t, \resp_t}_{t=1}^{n_0}$ drawn from $\mu$, $Q(\mu)$, failure probability $\delta$, and $R_{\max}$ such that $R_{\max} \geq \max_{a \in \mathcal{A}} |\langle a, \theta^* \rangle|$, 
and $n_0>\frac{32s^2(R_{\max}^2 + \sigma^2)H^2 (Q(\mu))}{\sigma^2} \log \frac{2d}{\delta}$,
then all the following items hold with probability at least $1-2\delta$:
\begin{enumerate}[label=(\roman*)]
    \item  $\|\hat{\theta}-\theta^*\|_\infty \leq 8 \sigma H(Q) \sqrt{\frac{\ln \frac{2d}{\delta}}{n_0}}$
    \item $\textrm{supp}(\hat{\theta})\subset \textrm{supp}({\theta}^*)$ so $\|\hat{\theta}-\theta^*\|_0 \leq s$
    \item $\|\hat{\theta}-\theta^*\|_1 \leq 8 s \sigma H(Q) \sqrt{\frac{\ln \frac{2d}{\delta}}{n_0}}$
\end{enumerate}
\end{corollary}
\begin{proof}
    Since $n_0$ is sufficiently large, from the  Theorem~\ref{thm: main bounds of estimator} with $R_0 = R_{\max}$ we can say that $\|\theta_0 -\theta^*\|_1 \leq \sigma$ with probability $1-\delta$. Applying Theorem~\ref{thm: main bounds of estimator} again with $R_0=\sigma$ we can get all (i), (ii), (iii) directly with probability $1-\delta$.
    The corollary follows from a union bound.
\end{proof}

\section{Proof of Proposition~\ref{prop:H2 vs Cmin} and Proposition~\ref{prop:H2 vs compat}}
\label{appendix: example of H2 and Cmin}
First, we will prove 
\begin{equation}
H_*^2 \leq \cC_{\min}^{-1}\leq d H_*^2.
\label{eqn:H-Cmin}
\end{equation}
For each of the two inequalities, We will give a tight example in the next subsection. 
\begin{proof}
For any positive definite matrix $Q\in\RR^{d \times d}$,
\begin{align}
    H^2 (Q)=\max_{i\in[d]}(Q^{-1})_{ii} =\max_{i\in[d]} e_i^\top Q^{-1} e_i \leq \max_{v\in\mathbb{S}^{d-1}} v^\top Q^{-1} v  =\lambda_{\max}(Q^{-1}) \leq \mathsf{tr}(Q^{-1}) \leq d H^2(Q)\label{eqn:ineq of H and Cmin}
\end{align}
Now, let the solution of the Eq. (\ref{def: H2}) and Eq. (\ref{def: Cmin}) as $\mu_H$ and $\mu_C$, respectively. Then, by the rightmost inequality of (\ref{eqn:ineq of H and Cmin}) we have $$\frac{1}{\mathcal{C}_{\min}} = \lambda_{\max} (Q(\mu_C)^{-1}) \leq \lambda_{\max} (Q(\mu_H)^{-1})\leq dH_*^2$$
and by the leftmost inequality of the (\ref{eqn:ineq of H and Cmin}) we have
$$ H_*^2 \leq H^2(Q(\mu_C)) \leq \lambda_{\max}(Q(\mu_C)^{-1}) = \frac{1}{\mathcal{C}_{\min}}$$
Therefore, the inequality part of the Proposition \ref{prop:H2 vs Cmin} holds. 
\end{proof}


\subsection{First equality condition analysis of Eq.~\eqref{eqn:H-Cmin}}

For the case when $\mathcal{C}_{\min}^{-1} = \Theta(H_*^2) $, consider $\cA = \{e_i|i=1, \cdots, d\}$; it can be seen that $H_*^2 = \mathcal{C}_{\min}^{-1}=d$. 

\subsection{Second equality condition analysis of Eq.~\eqref{eqn:H-Cmin}}
\label{example:worst case of Cmin and H2}




For the case when 
$
\mathcal{C}_{\min}^{-1} = \Theta(d H_*^2) 
$, consider $\cA = \cbr{a_1, \ldots, a_d}$, where 
\begin{align*}
    a_1 = & \frac{1}{\sqrt{d}}e_1\\
    a_i = & e_1 + \frac{1}{\sqrt{d}}e_i.
\end{align*} 
and we will calculate $H^2 (Q(\pi))$ and $\lambda_{\min}(Q(\pi))$ for the optimal sampling distributions $\pi$ to achieve $H^2_*$ and $\cC_{\min}$, respectively. 

\subsubsection{Prove that the optimal \texorpdfstring{$\pi$}{} satisfies \texorpdfstring{$\pi(a_2)=\pi(a_3)=\cdots = \pi(a_d)$}{}} \label{subsubsec: specialcase - equal is optimal}


We will first show that for both objectives $H^2 (Q(\pi))$ and $\lambda_{\min}(Q(\pi))$, 
there exists an optimal sampling distribution $\pi$
such that $\pi(a_2)=\pi(a_3)=\cdots = \pi(a_d)$.

Denote by $a:= \pi(a_1)$. Fix $a$.
For notational convenience, let $\pi(a_i):=b_i$ and $\bold{b}=(b_2, b_3 , \cdots, b_d) \in \mathbb{R}^{d-1}$. 

\paragraph{For $H^2(Q(\pi))$:} Then the covariance matrix $Q(\pi)$ (abbreviated as $Q$) has the following form: 
\begin{align}
    Q = \begin{bmatrix}
    \frac{a}{d} + \sum b_i & \frac{b_2}{\sqrt{d}} & \cdots & \frac{b_d}{\sqrt{d}}\\
    \frac{b_2}{\sqrt{d}} &  & & \\
    \vdots & & \frac{1}{d} \mathrm{diag}(\bold{b}) & \\
    \frac{b_d}{\sqrt{d}} &  & &
    \end{bmatrix}
\end{align}
After some calculation, one can get the determinant
$$\det(Q)=\frac{a(\Pi_{i=2}^d b_i)}{d^d}$$
and the cofactor 
\begin{align*} C_{ii} = \begin{cases}
    (\frac{\Pi_{i=1}^d b_i}{d^{d-1}}) & \text{if } i=1 \\
    (\frac{a}{d}+b_i)(\frac{\Pi_{s=2}^d b_s}{b_i d^{d-2}}) & \text{if } i=2, \cdots, d
\end{cases}
\end{align*}
and therefore

\begin{align*} (Q^{-1})_{ii} = \begin{cases}
    (\frac{d}{a}) & \text{if } i=1 \\
    (\frac{a}{d}+b_i)d^2 / (ab_i) & \text{if } i=2, \cdots, d
\end{cases}
\end{align*}
When $a$ is a fixed parameter, $(Q^{-1})_{ii} = \frac{d^2}{a} + \frac{d}{b_i}$ and therefore the $\arg \max_i (Q^{-1})_{ii} = \arg \min_i b_i$. Under the constraint $\sum_{i=2}^d b_i = 1-a$, the optimal solution is reached when $b_2 = b_3 = \cdots = b_d$. 

\paragraph{For $\lambda_{\min}(Q(\pi))$:} we will utilize symmetry of $\lambda_{\min}(Q(\pi))$. Note that $\lambda_{\min} (Q)$ is a concave function w.r.t $Q$. Suppose that the $({a}, b_2', b_3', \cdots , b_d') = \arg \max_{\pi} {\lambda_{\min}(Q(\pi))}$. Then from the symmetry, for any cyclic permutation $P$, all $({a}, b_{P^i (2)}', b_{P^i(3)}', \cdots , b_{P^i (d)}')$ $i=1, \cdots, d-1$ {also achieves} the maximum. Therefore, by Jensen's inequality, $$ {{\cC_{\min}}}=\frac{1}{d}\sum_{i=0}^{d-1} {\lambda_{\min} (Q(a, b_{P^i (2)}', b_{P^i(3)}', \cdots , b_{P^i (d)}'))} \leq {\lambda_{\min} (Q(a, \frac{1-a}{d-1}, \frac{1-a}{d-1}, \cdots , \frac{1-a}{d-1}) )}$$
Therefore, $(a, \frac{1-a}{d-1}, \frac{1-a}{d-1}, \cdots , \frac{1-a}{d-1})$ is also a maximizer of $\lambda_{\min}(Q(\pi))$.


{Therefore,} from now on, consider only the strategy $\pi$ that satisfies $\pi(a_2)=\pi(a_3)= \cdots = \pi (a_d) $ for this section, and let $a=\pi(a_1)$ and $b=\pi(a_2)$. Then $a+(d-1)b = 1$. Now the covariance matrix induced by $\pi$ is of the following form:
\begin{align}
    Q = \begin{bmatrix}
    \frac{a}{d} + (d-1)b & \frac{b}{\sqrt{d}} & \cdots & \frac{b}{\sqrt{d}}\\
    \frac{b}{\sqrt{d}} &  & & \\
    \vdots & & \frac{b}{d}I_{d-1} & \\
    \frac{b}{\sqrt{d}} &  & &
    \end{bmatrix}
\end{align}

\subsubsection{Calculating \texorpdfstring{$H_*^2$}{}}

One can calculate $\det(Q) = \frac{a}{d} (\frac{b}{d})^{d-1}$ (using again the cofactor method) and the cofactor
\begin{align*} C_{ii} = \begin{cases}
    (\frac{b}{d})^{d-1} & \text{if } i=1 \\
    (\frac{a}{d}+b)(\frac{b}{d})^{d-2} & \text{otherwise}
\end{cases}
\end{align*}
and therefore 

\begin{align*} (Q^{-1})_{ii} = \begin{cases}
    (\frac{d}{a}) & \text{if } i=1 \\
    (\frac{a}{d}+b)d^2 / (ab) & \text{otherwise}
\end{cases}
\end{align*}

{For $i = 2, \ldots, d$,} $(Q^{-1})_{ii}$ is always larger than $(Q^{-1})_{11}$, and by taking derivatives, the $b$ that minimizes the $(Q^{-1})_{22}$ is $\sqrt{\frac{d}{d-1}}-1 = \frac{1}{\sqrt{d-1}(\sqrt{d} + \sqrt{d-1})}$ , and the corresponding $(Q^{-1})_{22} = d(\sqrt{d}+\sqrt{d-1})^2 = \Theta(d^2)$. In this case, $a = d - \sqrt{d(d-1)}$ (close to 1/2).

\subsubsection{Calculating \texorpdfstring{$\cC_{\min}$}{}}


\chicheng{I have a somewhat simpler reasoning process that may perhaps give some general techniques that may be useful for the future:

First, note that for equation of the form $\lambda^2 - B \lambda + C = 0$ ($B, C > 0$), the smaller root is 
$\lambda^* = \frac{B - \sqrt{B^2 - 4C}}{2} = \frac{2C}{B + \sqrt{B^2 - 4C}} = \Theta(\frac{C}{B})$. This is because $B \leq B + \sqrt{B^2 - 4C} \leq 2B$. 

The first two roots of the characteristic equation are the roots of $\lambda^2 - B \lambda + C = 0$ for $B = \frac{1}{d} + \frac{d^2 - 2d + 2}{d} b$ and $C = \frac{b - (d-1)b^2}{d^2}$.
Therefore, the smaller root 
satisfies 
\[
\lambda^* = \Theta( \frac{C}{B} ) = 
\frac{ (b - (d-1)b^2)/d^2 }{ 1/d + (d^2-2d+2)b/d }
\leq 
\frac{ b / d^2 }{ b d } 
\leq 
\frac1 {d^3}.
\]
}
The characteristic function of the $Q$ is

$$ (\lambda^2 - B \lambda +C)(\lambda - \frac{b}{d})^{d-2}$$

where $B = \frac{1}{d} + \frac{(d^2 - 2d+2)b}{d} > 0$ and $C = \frac{b - (d-1)b^2}{d^2}>0$. Note that for equation of the form $\lambda^2 - B \lambda + C = 0$ ($B, C > 0$), the smaller root is 
$\lambda^* = \frac{B - \sqrt{B^2 - 4C}}{2} = \frac{2C}{B + \sqrt{B^2 - 4C}} = \Theta(\frac{C}{B})$. This is because $B \leq B + \sqrt{B^2 - 4C} \leq 2B$. 

Therefore, the smaller root of the quadratic equation $\lambda^2 - B\lambda + C = 0$
satisfies 
\[
\lambda^* = \Theta( \frac{C}{B} ) = 
\frac{ (b - (d-1)b^2)/d^2 }{ 1/d + (d^2-2d+2)b/d }
\leq 
\frac{ b / d^2 }{ b d } 
\leq 
\frac1 {d^3}.
\]

\ja{The eigenspectrum of $Q$ is 
\begin{align*}
    \lambda_1, \lambda_2 &= \frac{[\frac{1}{d} + \frac{(d^2 - 2d+2)b}{d}]\pm\sqrt{[\frac{1}{d} + \frac{(d^2 - 2d+2)b}{d}]^2 - \frac{4b - 4 (d-1)b^2}{d^2}}}{2}\\
    \lambda_3&=\lambda_4 = \cdots = \lambda_d = \frac{b}{d}
\end{align*}

Therefore, $\frac{1}{\lambda_{min}}= \max (\frac{d}{b}, \frac{d}{2b} (\frac{1+(d^2 -2d+2)b + \sqrt{(1+(d^2 -2d+2)b)^2 - 4b + 4 (d-1)b^2}}{1-(d-1)b}))$.

Note that for the optimal $\pi$, $a > 0$; otherwise $Q(\pi)$ is not invertible. Therefore, $b = \frac{1-a}{d-1} < \frac1{d-1}$.
Hence, one can check that $\sqrt{(1+(d^2 -2d+2)b)^2 - 4b + 4 (d-1)b^2} \geq 1$, so $\frac{1}{\lambda_{min}}= \frac{d}{2b} (\frac{1+(d^2 -2d+2)b + \sqrt{(1+(d^2 -2d+2)b)^2 - 4b + 4 (d-1)b^2}}{1-(d-1)b})$. Now we investigate $\arg\min_{b} \frac{1}{\lambda_{min}(b)}$.

Recall that 
$a+(d-1)b=1$, $b<1/(d-1)$. Thus, $-4b + 4(d-1)b^2 < 0$. Therefore, the big square root term is smaller than $(1+(d^2 -2d+2)b)$, which means if we let $f(b)=\frac{d (1+ (d^2 -2d+2)b)}{2b(1-(d-1)b)}$, then $$ f(b) \leq \frac{1}{\lambda_{min}(b)} \leq 2f(b)$$ 
It remains to calculate $\min_b f(b)$. A few derivative calculations show that $b^*= \arg \min_b f(b) = \frac{-(d-1) + \sqrt{d^3 -1}}{d^3-1-(d-1)^2} = \Theta (d^{-3/2})$ and $\min_b f(b) = \Theta(d^3)$.}

\subsubsection{Lower bound of $\frac{1}{\phi_0^2(Q(\pi),s)}$}

It is difficult to directly calculate the compatibility constant of $Q(\pi)$, but we can bound it using the diagonal entries of $Q(\pi)$. Note that 
$$\phi^2 = \min_{S \subset [d]}\min_{v\in\mathcal{C}_S} \frac{s v^\top Q v}{\|v_S\|_1^2} \leq s\min_{v \in \{e_i \}_{i=1}^d} \frac{v^\top Q v}{1} = s\min_{i \in [d]}(Q_{ii})$$ and therefore $\frac{1}{\phi^2} \geq \frac{1}{\min_{i \in [d]} sQ_{ii}}$. 

We will use the same notation in \ref{subsubsec: specialcase - equal is optimal}: denote by $a:= \pi(a_1)$ and let $\pi(a_i):=b_i$ and $\bold{b}=(b_2, b_3 , \cdots, b_d) \in \mathbb{R}^{d-1}$. Then the covariance matrix $Q(\pi)$ (abbreviated as $Q$) has the following form: 
\begin{align}
    Q = \begin{bmatrix}
    \frac{a}{d} + \sum b_i & \frac{b_2}{\sqrt{d}} & \cdots & \frac{b_d}{\sqrt{d}}\\
    \frac{b_2}{\sqrt{d}} &  & & \\
    \vdots & & \frac{1}{d}Diag(\bold{b}) & \\
    \frac{b_d}{\sqrt{d}} &  & &
    \end{bmatrix}
\end{align}

From the basic constraint $a+\sum_{i=2}^d b_i=1$, $\min_{i \in [d]} Q_{ii} = \min_{i \in \{2, \cdots, d\}} \frac{b_i}{d} =O(\frac{1}{d^2})$. Therefore $\frac{1}{\phi^2} = \Omega(d^2/s)$.
This means even for the best case of the compatibility constant cannot beat the recovery bound of \popart {for} this action set. 

\section{Proofs for Sparse Linear Bandits}
\label{sec:proof-slb}

\subsection{Proof of Theorem \ref{thm:with minimum signal}}

\begin{proof}
From the (i) in Corollary \ref{cor:warm-popart-fullversion}, when $n_2 = \frac{256\sigma^2 H_*^2}{m^2} \log \frac{d}{\delta}$, with probability at least $1-2\delta$, $$\|\hat{\theta}-\theta^*\|_\infty < 8\sigma\sqrt{\frac{H_*^2}{n_2} \log \frac{2d}{\delta}} = \frac{m}{2}$$
Therefore, with probability at least $1-2\delta$, for any index $i \in \textrm{supp}(\theta^*)^C$, $\hat{\theta}_i =0$, and for any index $ j \in \textrm{supp}(\theta^*)$, $|\hat{\theta}_j|> |\theta_j^*|-\frac{m}{2}>0$. Thus, $\textrm{supp}(\theta^*)=\textrm{supp}(\hat{\theta})$ with probability at least $1-2\delta$. 
After that, we use the following result about the {phased elimination} \cite{lattimore18bandit}:
\begin{theorem}(\citet{lattimore18bandit}, Theorem 22.1)
The $n$-step regret of phased elimination algorithm satisfies
$$\textrm{Reg}_n \leq C \sigma
\sqrt{nd \log(|\cA| n)}$$
for an appropriately chosen universal constant $C > 0$.
\end{theorem}
\end{proof}

\section{Proof of Lower Bound}
\label{sec:proof-lb}

\chicheng{It just occurs to me that $C_{\min}$ and $C_{\min}(\cA)$ are both used. Can we use one consistently throughout the paper?}
\chicheng{Same comment applies to $\Reg_n$ vs. $\Reg(n)$} \ja{I added some abuse of notation in 'Notations' part}
\chicheng{I am sorry, I am not seeing the additional clarification you added.. Will check with you when we meet.}
\ja{TODO: Change the condition of theorem below as $\kappa \in (0, \frac{1}{2s}), d\geq (s+1)^2 , n \in (8\kappa s^2, ...)$}

{
In this section, we prove Theorem~\ref{thm:lower}. We start with a restatement of it.



\begin{theorem}(Restatement of the Theorem \ref{thm:lower}) 
\label{thm:lower-restated}
For any algorithm, any $s, d, \kappa$ that satisfies $s>2000, \kappa \in (0,1), n>8\kappa s^2, d \geq \max(n^{1/3} s^{4/3} \kappa^{-4/3},(s+1)^2)$, there exists a linear bandit environment with an action set $\cA$ and a $s$-sparse $\theta \in \RR^{d}$, such that $\mathcal{C}_{\min}(\cA)^{-1} \leq \kappa^{-2}$, $R_{\max} \leq 2$, $\sigma = 1$, and 
\[
\textrm{Reg}_n \geq \Omega( \kappa^{-2/3} s^{2/3} n^{2/3})~.
\]
\end{theorem}

In the lower bound instance that establishes Theorem~\ref{thm:lower-restated}, we will prove that ${2\kappa^{-2}} \geq \cC_{\min}(\cA)^{-1} \geq H_*^2 $ (see Section \ref{subsubsec:kappa vs cmin}), and conclude that our $\tilde{O}( H_*^{2/3} s^{2/3} n^{2/3} )$ regret upper bound of Algorithm \ref{alg:etc-sparse} has a matching lower bound and conclude that the algorithm and the lower bound are both optimal in this setting.

For convenience, throughout the rest of this section,
we prove the following slight variant of Theorem~\ref{thm:lower-restated}, where the  dimensonality is $d+1$ as opposed to $d$, and the sparsity is $2s+1$ as opposed to $s$; note that the changes of these parameters do not affect the orders of the regret bounds in terms of them.

\begin{theorem} 
\label{thm:lower-restated-2s}
For any algorithm, any $s, d, \kappa$ that satisfies $s>1000$ and is a multiple of 4, $\kappa \in (0,1), n>8\kappa s^2, d \geq \max(n^{1/3} s^{4/3} \kappa^{-4/3},(s+1)^2)$, there exists a linear bandit environment an action set $\cA$ and a $(2s+1)$-sparse $\theta \in \RR^{d+1}$, such that $\mathcal{C}_{\min}(\cA)^{-1} \leq 2\kappa^{-2}$, $R_{\max} \leq 2$, $\sigma = 1$, and 
\[
\textrm{Reg}_n \geq \Omega( \kappa^{-2/3} s^{2/3} n^{2/3})~.
\]
\end{theorem}



\paragraph{Construction} Following the standard minimax lower bound and hypothesis testing terminology, 
we will often refer to an underlying reward predictor $\theta \in \RR^{d+1}$ as a \emph{hypothesis}. 
Let $$ \Theta_s = \Big\{ \theta \in \RR^{d+1}| \theta_i \in \{-\epsilon, 0, \epsilon\} \text{ for } i \in [d], \theta_{d+1}=-1, \|\theta\|_0=s+1\Big\},$$
where $\epsilon= \kappa^{-2/3}s^{-1/3}n^{-1/3}$.
We will use $\Theta_{s}$ and $\Theta_{2s}$ as our hypothesis space throughout the proof.

We construct a low-regret action set $\cI$ and an informative action set $\cH$ as follows:

\begin{align*}
    \cI &= \Big\{ x \in \RR^{d+1} | x_j \in \{-1,0,1\} \text{ for } j \in [d], \|x\|_0 = 2s, x_{d+1}=0 \Big\}\\
    \cH &= \Big\{ x \in \RR^{d+1} | x_j \in \{-\kappa,\kappa\} \text{ for } j \in [d], |\sum_{j=1}^d x_j| \leq {\kappa} \sqrt{2d \ln 2d}, x_{d+1}=1 \Big\}
\end{align*}
where $\kappa \in (0,1)$ is a constant. The action set is the union $\cA = \cI \cup \cH$.


Our linear bandit environment parameterized by $\theta \in \RR^{d+1}$ is defined as: given action taken $A_t$, its reward $r_t = \inner{\theta}{A_t} + \eta_t$, where $\eta_t \sim N(0, 1)$ is an independently drawn standard Gaussian noise. 
Note that by construction, $\eta_t$ is $\sigma^2$-subgaussian with $\sigma = 1$.


\paragraph{Notations} In this section, we will use $\bA = (A_1, \cdots, A_n) \in \cA^n$ as the random variable about the history of actions. For $\ba {= (a_1, \ldots, a_n)} \in \cA^n$, let $T(\cH; \ba) = \sum_{t=1}^n \one (a_t \in \cH)$, {which represents the total number of pulls of arms in $\cH$ in the learning process}. For brevity of notation, we will write the random variable $T(\cH; \bA)$ as $T(\cH)$, and $\cC_{\min}(\cA)$ as $\cC_{\min}$ throughout this section. Let $\mathsf{Sub}_x = \{S \subset [d] | |S|=x \}$, the set of subsets of $[d]$ which has $x$ elements. In subsequent proofs, given a bandit algorithm \alg and an bandit environment $\theta$, we use $\PP_{\theta,\alg}$ and $\EE_{\theta, \alg}$ to denote probability and expectation under the probability space induced by the interaction history between them. For any set of indices $S \subset [d+1]$, let $\sym(S)$ be the symmetric group of the set $S$ {(i.e. the collection of all permutations over $S$)}, and let $\Pi_S = \cbr{ \sigma \in \sym([d+1]): {\sigma(j)=j \text{ for all } j \in [d+1]\backslash S} }$ be the set of permutations which permutes only the indices in $S$, and let $\Pi_{a:b} = \Pi_{\{a, a+1, \cdots, b\}}$.

\paragraph{Structure of the section} Here is the high-level idea of the proof structure. 
\begin{itemize}
     \item Reduction {to symmetric algorithms using algorithmic symmetrization} (Section \ref{subsec:symmetric reduction}) : First, we will prove that the regret lower bound of symmetric sparse linear bandit algorithms (see Definition~\ref{def:symmetry}) is also the lower bound of the general sparse linear bandit algorithms (Lemma \ref{lem:algs then alg}). Keen readers may note that our action set construction is symmetric except for the $(d+1)$-th coordinate, and this is for exploiting the symmetry. By {focusing} on proving lower bounds for symmetric algorithms, we can exploit the favorable combinatorial properties of our action spaces to establish tighter lower bounds. 
    \item Count the number of mistakes (Section \ref{subsubsec: count mistake}): Next, we will prove the core proposition of the lower bound proof, Proposition \ref{claim: main lower bound}. This proposition can be summarized as, `the learning agent has to pull \edit{}{a} sufficiently large number of arms in $\cH$ (informative actions with high regret) to make less mistakes', where  `mistakes' refers to coordinates in the support of $\theta$ that has not been `hit' sufficiently by the agent via  pulling the low-regret arms $\cI$ (See Equation~\eqref{eqn:m-theta} for a formal definition).
    This implies an inherent tension between pulling informative, high regret arms $\cH$ and pulling low regret arms $\cI$, which eventually leads to the desired lower bound in  Theorem~\ref{thm:lower}.
    \item Lower bound on symmetric algorithms (Section \ref{subsubsec: proof of main prop}) : Now it remains to show the proof of Proposition \ref{claim: main lower bound}. Here, to improve the $\Omega(s^{1/3} n^{2/3})$ regret lower bound proved by~\citet{hao2020high} to $\Omega(s^{2/3} n^{2/3})$, 
    we deviate from their usage of Bretagnolle-Huber inequality for binary hypothesis testing, and 
    take a novel combination of various techniques such as a change of measure technique, combinatorial
    calculation by utilizing symmetry (Claim~\ref{claim: KL divergence}). 
\end{itemize}
}



\subsection{Algorithmic symmetrization: reducing lower bounds for general algorithms to symmetric algorithms} \label{subsec:symmetric reduction}

In this section, we show how {proving} a lower bound for generic algorithms can be reduced to that of permutation-symmetric (abbrev. symmetric) (augmented) algorithms (Definition~\ref{def:symmetry}), specifically Lemma~\ref{lem:algs then alg}. 
To introduce symmetric  algorithms, let us first define some useful terminology. 
{We first define a frequent coordinate set, which is the set of coordinates that are frequently 'hit' by low-regret arm pulls ($\cI$.)}



%


\begin{definition}[Frequent coordinate set]
Let ${\cU(\ba)} = \{U \in \mathsf{Sub}_{d/2}| \forall i\in U, \sum_{t=1}^n |a_{ti}|\one(a_t \in \cI) \ge \kmax{(d/2)} \{\sum_{t=1}^n |a_{tj}|\one(a_t \in \cI)\}_{j=1}^{d}  \}$ where $\kmax{k} S$ for a set $S \subseteq \RR$ is the $k$-th largest {element} of $S$. 
\end{definition}

{We also define coordinate-selection bandit algorithm which outputs top $d/2$-coordinates that are most frequently hit.}

\begin{definition}[]({Coordinate-selection bandit algorithm;} Coordination of an algorithm).
\label{def:augment}

\begin{enumerate}
\item Define a coordinate-selection bandit algorithm $\mathsf{B}$ as: at time {step} $t$, choose action $A_t$ based on its historical observations $(A_s, r_s)_{s=1}^{t-1}$; 
finally it outputs $\hat{S} \in \cU(\bA)$.
In other words, all elements in $\hat{S}$ are among the top $\frac{d}{2}$ most frequently chosen coordinates (including ties) when restricted to arm pull history on $\cI$.

\item Given a bandit algorithm \alg, define its coordination \alga as:
at time step $t$, use \alg to output $A_t$ based on all historical observations $(A_s, r_s)_{s=1}^{t-1}$; finally, output $\hat{S} \in \cU(\bA)$ with the lowest dictionary index\footnote{{The choice of dictionary order here is merely for concreteness; the proof would also go through if we break ties in other orders.}}.
In other words, the elements in $\hat{S}$ are the top $\frac{d}{2}$ most frequently chosen coordinates when restricted to arm pull history on $\cI$. 
\end{enumerate}
\end{definition}

 With this notation, for any coordinate-selection bandit algorithm, its output $\hat{S} \in \cU(\bA)$ with probability 1.
As a result, we will mainly focus on $(\ba, U) \in \cA^n \times \mathsf{Sub}_{d/2}$ such that $U \in \cU(\ba)$, but for the lemmas we keep the generality and consider any $U \subset [d]$. 

\begin{remarks}
From the above definitions, it can be readily seen that \alg's coordination, \alga, is a valid coordinate-selection bandit algorithm. {However, a coordinate-selection bandit algorithm does not need to break ties in dictionary order. }
\end{remarks}

\begin{remarks}
\alga outputs $\hat{S}$ by breaking ties in dictionary order. While this breaks symmetry by favoring coordinates with lower indices, as we will see in our reduction proof (proof of Lemma~\ref{lem:algs then alg}), we do not require \alga to be symmetric (we will define symmetry momentarily in  Definition~\ref{def:symmetry}); instead, we will work on a symmetrized version of \alga (Definition~\ref{def:symmetrized-algap}). 
\end{remarks}

\begin{definition}[Permutaion over sets of coordinates, and vectors in $\RR^{d+1}$]
Given a permutation $\sigma \in \Pi_{1:d}$:
\begin{itemize}
    
    \item For a subset of coordinates $S \subset [d]$, define $\sigma(S) \subset [d]$ as $\sigma(S):= \cbr{ \sigma(i): i \in S }$.

    \item For vector $v \in \RR^{d+1}$, define $\sigma(v) = ( v_{\sigma^{-1}(1)}, \ldots, v_{\sigma^{-1}(d+1)} ) = P_\sigma v \in \RR^{d+1}$ as the permuted version of $v$ using $\sigma$, where $P_\sigma = (e_{\sigma(1)}, \ldots, e_{\sigma(d+1)}) \in \RR^{(d+1) \times (d+1)}$ is the  permutation matrix\footnote{Here we use $\sigma$'s row representation.} induced by $\sigma$ and $e_j$ denotes $j$-th standard basis. Note that for every $i \in [d+1]$, 
    $\sigma(e_i) = e_{\sigma(i)}$.
    
    \item For sequence of actions $\ba = (a_1, \ldots, a_n) \in \cA^n$, define $\sigma(\ba) = (\sigma(a_1), \ldots, \sigma(a_n))$ as its  permuted version using $\sigma$.
\end{itemize}
\end{definition}

Intuitively, $\sigma( \cdot )$ ``moves'' the $i$-th entry of input vector $v$ to the vector's $\sigma(i)$-th coordinate. 
We will frequently apply the above vector permutation operation in our subsequent proofs, where the vector $v \in \RR^{d+1}$ are often taken as actions $A_t$ or hypotheses (underlying reward predictors) $\theta$.

\begin{definition}[Permutation-invariant action space]
An action space $\cA$ is said to be permutation-invariant, if for any $\pi \in \Pi_{1:d}$, 
\[
\pi(\cA) := \cbr{ \pi(a): a \in \cA } = \cA.
\]
\end{definition}
By our construction in the beginning of Section~\ref{sec:proof-lb}, our action space $\cA = \cI \cup \cH$ is permutation invariant.

Now we are ready to define symmetric coordinate-selection bandit algorithms, a special class of bandit algorithms we will focus on.

\begin{definition}[Symmetric coordinate-selection bandit algorithm]
\label{def:symmetry}
a coordinate-selection bandit algorithm $\mathsf{B}$ {over a permutation invariant action space $\cA$} is said to be \emph{symmetric}, if for any $\sigma \in \Pi_{1:d}$ and any $\ba = (a_1, \ldots, a_n) \in \cA^n$ and $U \subset [d]$,
    \[
    \PP_{\theta, \mathsf{B}}(\bA = \ba, \hat{S} = U)
    = 
    \PP_{\sigma(\theta), \mathsf{B}}(\bA = \sigma(\ba), \hat{S} = \sigma(U) ).
    \]
\end{definition}

Note that the above permutation symmetry notion is slightly different from~\citet{simchowitz2017simulator} -- here we only consider permutations in $\Pi_{1:d}$, i.e., over the first $d$ coordinates (out of all $d+1$ coordinates), whereas~\citet{simchowitz2017simulator} consider permutations over all coordinates (arms).

For symmetric coordinate-selection bandit algorithms, we have the following elementary property. {Hereafter, all proofs are deferred to Section~\ref{sec:deferred}.}

\begin{lemma}
\label{lem:algas}
For {any} symmetric coordinate-selection bandit algorithm $\mathsf{B}$, any $\sigma \in \Pi_{1:d}$ and any function $f: \cA^n \times 2^{[d]} \to \RR$,
    \[
    \EE_{\theta, \mathsf{B}}\sbr{ f( \bA, \hat{S}) }
    = 
    \EE_{\sigma(\theta), \mathsf{B}}\sbr{ f( \sigma^{-1}(\bA), \sigma^{-1}(\hat{S}) ) }
    \]
\end{lemma}

\begin{definition}[Permuted augmented algorithm]
For a coordinate-selection bandit algorithm $\mathsf{B}$ on a permutation-invariant action space $\cA$, 
and a permutation $\pi \in \Pi_{1:d}$, 
define its $\pi$-permuted version $\mathsf{B}\pi$ as: first permute the $[d]$ coordinates using $\pi$, and run $\B$ with the permuted coordinates.
Formally, at every time step $t$:
\begin{itemize}
    \item 
    $\B$ outputs some action $A_t' \in \cA$, and $\B \pi$ accordingly outputs action $A_t = \pi^{-1}(A_t') \in \cA$
    \item  Receives reward $r_t = \inner{\theta}{A_t} + \eta_t
    =
    \inner{\pi(\theta)}{\pi(A_t)} + \eta_t
    =
    \inner{\pi(\theta)}{A_t'} + \eta_t$
\end{itemize}
Finally, $\B$ outputs $\hat{S}'$, and $\B \pi$ outputs $\hat{S} = \pi^{-1}(\hat{S}')$.
\end{definition}




The following lemma follows straightforwardly from the definition of $\B\pi$:
\begin{lemma}
\label{lem:algapi}
\begin{itemize}
    \item For any $\ba = (a_1, \ldots, a_n) \in \cA^n$ and $U \subset [d]$,
    \[ 
    \PP_{\theta,\Bpi}(\bA = \ba, \hat{S} = U)
    = 
    \PP_{\pi(\theta),\B}(\bA = \pi(\ba), \hat{S} = \pi(U)).
    \]
    \item For any function $f: \cA^n \times 2^{[d]} \to \RR$,
    \[
    \EE_{\theta, \Bpi}\sbr{ f(\bA, \hat{S}) }
    = 
    \EE_{\pi(\theta), \B}\sbr{ f( \pi^{-1}(\bA), \pi^{-1}(\hat{S}) ) }.
    \]
\end{itemize}
\end{lemma}

\begin{definition}
\label{def:symmetrized-algap}
For a coordinate-selection bandit algorithm $\B$ on a permutation-invariant action space $\cA$, 
define its symmetrized version 
$\BP$ as: first, choosing $\pi$ uniformly at random from $\Pi_{1:d}$, then, run $\Bpi$ on the bandit environment for $n$ rounds.
\end{definition}

\begin{lemma}
\label{lem:algap}
We have the following: 
\begin{enumerate}
    \item $\PP_{\theta, \BP}\del{ \cdot } = \frac{1}{|\Pi_{1:d}|} \sum_{\pi \in \Pi_{1:d}} \PP_{\theta,\Bpi}\del{ \cdot }$, and
    $\EE_{\theta,\BP}[\cdot] = \frac{1}{|\Pi_{1:d}|} \sum_{\pi \in \Pi_{1:d}} \EE_{\theta,\Bpi}[\cdot]$.
    \item $\BP$ is a symmetric coordinate-selection bandit algorithm. 
\end{enumerate}
\end{lemma}

The definition below formalizes the (pseudo-)regret notion under a specific hypothesis, which provides useful clarifications 
when using the averaging hammer to argue regret lower bounds.
\begin{definition}
\label{def:reg-theta}
Define 
\[
\Reg(\bA,\theta)
=
n \cdot \max_{a \in \cA} \inner{\theta}{a} - \sum_{t=1}^n \inner{\theta}{A_t}
\]
as the pseudo-regret of a sequence of actions $\bA = (A_t)_{t=1}^n$ under hypothesis $\theta$.
\end{definition}

The main result of this section is the following lemma that reduces proving lower bounds for general algorithms to proving lower bounds for symmetric augmented algorithms.
\begin{lemma}[Algorithmic symmetrization lemma] \label{lem:algs then alg}
If for all symmetric coordinate-selection bandit algorithms $\B$, there exists some $\theta \in \Theta_s \cup \Theta_{2s}$ such that 
$\EE_{\theta, \B} \sbr{ \Reg(\bA, \theta) }
\geq R$, 
then,
for all bandit algorithms \alg, 
there exists some $\theta' \in \Theta_s \cup \Theta_{2s}$ such that 
$\EE_{\theta', \alg} \sbr{ \Reg(\bA, \theta') }
\geq R$.
\end{lemma}

In view of this lemma, in Section~\ref{sec:lb-symmetric}, we focus on showing regret lower bounds on symmetric coordinate-selection bandit algorithms under hypotheses in $\Theta_s \cup \Theta_{2s}$. 

\subsubsection{Deferred Proofs}
\label{sec:deferred}

\begin{proof}[Proof of Lemma~\ref{lem:algas}]
For any $f: \cA^n \times 2^{[d]} \to \RR$, 
\begin{align*}
    \EE_{\theta, \B}\sbr{ f( \bA, \hat{S}) }
    &= \sum_{(\ba, U) \in \cA^n \times 2^{[d]}} \PP_{\theta, \B}(\bA = \ba, \hat{S} = U) f( \ba, U)
    \tag{definition of expectation}
    \\
    &= \sum_{(\ba, U) \in \cA^n \times 2^{[d]}} \PP_{\sigma(\theta), \B}(\bA = \sigma(\ba), \hat{S} = \sigma(U)) f( \ba, U) \tag{symmetry}\\
    &= \sum_{(\ba, U) \in \cA^n \times 2^{[d]}} \PP_{\sigma(\theta), \B}(\sigma^{-1}(\bA) = \ba, \sigma^{-1}(\hat{S}) = U) f( \ba, U)
    \tag{algebra}
    \\
    &=\EE_{\sigma(\theta), \B}\sbr{ f( \sigma^{-1}(\bA), \sigma^{-1}(\hat{S}) ) } \tag{definition of expectation}
\end{align*}
\end{proof}


\begin{proof}[Proof of Lemma~\ref{lem:algapi}]
For the first item, 
denote by $\bA' = (A_1', \ldots, A_n')$; 
for any $\ba = (a_1, \ldots, a_n) \in \cA^n$ and $U \subset [d]$,

\begin{align*}
    \PP_{\theta,\Bpi}(\bA = \ba, \hat{S} = U)
    &= \PP_{\theta,\Bpi}(\pi^{-1}(\bA') = \ba, \pi^{-1}(\hat{S}') = U)
    \tag{definition of $\bA'$}
    \\
    & = 
    \PP_{\theta,\Bpi}( \bA' = \pi(\ba), \hat{S}' = \pi(U) ) 
    \tag{algebra}
    \\
    &=\PP_{\pi(\theta),\B}(\bA = \pi(\ba), \hat{S} = \pi(U))
    \tag{switching to $\B$'s perspective}
\end{align*} 

The second item is the direct consequence of the first item by the following calculation.
\begin{align*}
    \EE_{\theta, \Bpi}\sbr{ f(\bA, \hat{S}) }
    &= \sum_{(\ba, U) \in \cA^n \times 2^{[d]}} \PP_{\theta, \Bpi}(\bA = \ba, \hat{S} = U) f( \ba, U)
    \tag{definition of expectation}
    \\
    &= \sum_{(\ba, U) \in \cA^n \times 2^{[d]}} \PP_{\pi(\theta), \B}(\bA = \pi(\ba), \hat{S} = \pi(U)) f( \ba, U) \tag{the first item}\\
    &= \sum_{(\ba, U) \in \cA^n \times 2^{[d]}} \PP_{\pi(\theta), \B}(\pi^{-1}(\bA) = \ba, \pi^{-1}(\hat{S}) = U) f( \ba, U) \tag{algebra}\\
    &= \EE_{\pi(\theta), \B}\sbr{ f( \pi^{-1}(\bA), \pi^{-1}(\hat{S}) ) }.
    \tag{definition of expectation}
\end{align*}


\end{proof}
    
\begin{proof}[Proof of Lemma~\ref{lem:algap}]
The first item follows from the definition of $\BP$. 

For the second item, for any permutation $\sigma \in \Pi_{1:d}$ and action history $\ba \in \mathcal{A}^n$,
\begin{align*}
    \PP_{\th,\BP} (\bA = \ba, \hat{S} = U) 
    &= 
    \frac{1}{|\Pi_{1:d}|}\sum_{\pi \in \Pi_{1:d}}\PP_{\th,\Bpi} (\bA = \ba, \hat{S} = U) 
    \tag{the first item}
    \\
    &= 
    \frac{1}{|\Pi_{1:d}|}\sum_{\pi \in \Pi_{1:d}}\PP_{\pi(\th),\B} (\bA = \pi(\ba), \hat{S} = \pi(U)) 
    \tag{Lemma~\ref{lem:algapi}}
    \\ 
    &= 
    \frac{1}{|\Pi_{1:d}|}\sum_{\pi \in \Pi_{1:d}}\PP_{\pi \circ \sigma(\th), \B} (\bA = \pi \circ \sigma(\ba), \hat{S} = \pi \circ \sigma(U)) 
    \tag{*}
    \\
    &= 
    \frac{1}{|\Pi_{1:d}|}\sum_{\pi \in \Pi_{1:d}}\PP_{ \sigma(\th),\Bpi} (\bA = \sigma(\ba), \hat{S} = \sigma(U)) 
    \tag{Lemma~\ref{lem:algapi}}
    \\
    &=\PP_{\sigma (\th), \BP} (\bA = \sigma (\ba), \hat{S}=\sigma(U)),
    \tag{the first item}
\end{align*}
where in step (*), we use the observation that for any $\sigma \in \Pi_{1:d}$, $\cbr{\pi \circ \sigma: \pi \in \Pi_{1:d}} = \Pi_{1:d}$.
\end{proof}

\begin{proof}[Proof of Lemma~\ref{lem:algs then alg}]
{Given any bandit algorithm \alg, denote by \alga its coordination (Definition~\ref{def:augment}), and denote by \algap the symmetrized version of \alga (Definition~\ref{def:symmetrized-algap}).
Since \algap is a symmetric augmented algorithm,} 
by assumption, 
we have, there exists some $\theta \in \Theta_{s} \cup \Theta_{2s}$, 
\begin{align*}
    R & \leq  \EE_{\theta, \algap} [
    \Reg (\bA, \theta)] \\
    &= \frac{1}{|\Pi_{1:d}|}\sum_{\pi \in \Pi_{1:d}}\EE_{\theta, \algapi} [\Reg (\bA, \theta)] 
    \tag{Lemma~\ref{lem:algap}}
    \\
    & = \frac{1}{|\Pi_{1:d}|}\sum_{\pi \in \Pi_{1:d}}\EE_{\theta, \algapi} [n \cdot \max_{a \in \cA} \inner{\theta}{a} - \sum_{t=1}^n \inner{\theta}{A_t}] \tag{Definition~\ref{def:reg-theta}} \\
    & = \frac{1}{|\Pi_{1:d}|}\sum_{\pi \in \Pi_{1:d}}\EE_{\pi(\theta), \alga} [n \cdot \max_{a \in \cA} \inner{\theta}{a} - \sum_{t=1}^n \inner{\theta}{\pi^{-1}(A_t)}] 
    \tag{Lemma~\ref{lem:algapi}}
    \\
    & = \frac{1}{|\Pi_{1:d}|}\sum_{\pi \in \Pi_{1:d}}\EE_{\pi(\theta), \alga} [n \cdot \max_{a \in \cA} \inner{\pi(\theta)}{a} - \sum_{t=1}^n \inner{\pi(\theta)}{A_t}] \tag{$\langle a, \pi^{-1}(b) \rangle = \inner{\pi(a)}{b}$, and $\cA$'s permutation invariance} \\
    &=\frac{1}{|\Pi_{1:d}|}\sum_{\pi \in \Pi_{1:d}}\EE_{\pi(\theta), \alga} [\Reg ( \bA, \pi(\theta))] 
    \tag{Definition~\ref{def:reg-theta}}
    \\
    &=\frac{1}{|\Pi_{1:d}|}\sum_{\pi \in \Pi_{1:d}}\EE_{\pi(\theta), \alg} [\Reg ( \bA, \pi(\theta))]
    \tag{\alg and \alga take the same action sequence}
\end{align*}
By the {probabilistic method}, there exists $\pi \in \Pi_{1:d}$ which satisfies $\EE_{\pi(\theta), \alg} [\Reg ( \bA, \pi(\theta))]\geq R$, and this $\pi(\theta) \in \Theta_s \cup \Theta_{2s}$ is the desired $\theta'$ in Lemma~\ref{lem:algs then alg}. 
\end{proof}

\subsection{Lower bound against symmetric algorithms}
\label{sec:lb-symmetric}

\subsubsection{Counting the number of mistakes} \label{subsubsec: count mistake}

{From now on, 
by Lemma~\ref{lem:algs then alg}, 
we will focus on proving regret lower bound for any symmetric augmented algorithm $\B$ (Definition~\ref{def:symmetry}). 
For the brevity of notation,
we omit the $\B$ {subscripts} from $\PP$ and $\EE$. {We view the final output $\hat{S}$ as an estimator of $\supp(\theta)$, and } 
define the $M_\theta (\hS)$, the {number of false negative mistakes} respect to $\theta$ as follows:
\begin{align}
  {M_\th(\hS)} &:= \supp(\th) \sm \hS
  \label{eqn:m-theta}
\end{align}

Let ${s}$ be multiple of 4. For $\xi = \frac{1}{4}$, we like to show the following claim.}
\chicheng{Can we plug in the numerical value of $\xi$ throughout the proof? It is a constant anyways.}

\begin{proposition}\label{claim: main lower bound}
    If $P_\th( | M_\th(\hS) | \ge s/4) \le \xi$ for all $\theta\in \Theta_s \cup \Theta_{2s}$, then $\exists \theta' \in \Theta_s$ such that $\EE_{\th'} [T(\cH) ] \ge \Omega(\frac{1}{\kappa^2\epsilon^2})$
\end{proposition}

\begin{remarks}
~\cite{hao2020high} shows a weaker version of this proposition, which (essentially) shows that if we would like to suffer a regret $\leq \Omega(n s \epsilon)$ under $\Theta_{s} \cup \Theta_{2s}$, the number of pulls to the informative arms needs to be at least $\Omega(\frac{1}{s \kappa^2 \epsilon^2})$. 
As our main technical contribution, we improve~\cite{hao2020high}'s lower bound on $\EE_{\th'} [T(\cH)]$ by a factor of $s$. Note that our definition of $\cH$, the set of informative arms, is slightly different from~\cite{hao2020high}. 
\end{remarks}

Given the above proposition, we are now ready to prove Theorem~\ref{thm:lower-restated-2s}.

\begin{proof}[Proof of Theorem~\ref{thm:lower-restated-2s}]
{We begin by showing that for all 
$\theta \in \Theta_s \cup \Theta_{2s}$, 
\[ 
\EE_\theta \sbr{ \Reg_n } \geq \frac{\epsilon s n}{8} \PP_\th( | M_\th(\hS) | \ge s/4).
\]
It suffices to show that, if $|M_\th(\hS) | \ge s/4$, then $\Reg_n \geq \frac{\epsilon s n}{8}$. 
Indeed, 
\begin{align*}
\Reg_n 
&= 
\sum_{t=1}^n \langle \theta, a^*-A_t \rangle  \\
&= 
\epsilon |\supp(\theta)| n  - \sum_{t=1}^n \langle \theta, A_t \rangle \indic(A_t \in \cI) - \sum_{t=1}^n \langle \theta, A_t \rangle \indic(A_t \in \cH)  \\
& \geq 
\epsilon |\supp(\theta)| n  - \sum_{t=1}^n \langle \theta, A_t \rangle \indic(A_t \in \cI) \\
& =  
\epsilon |\supp(\theta)| n  - 
\epsilon \sum_{i \in \supp(\theta) \cap \hat{S}}
\sum_{t=1}^n A_{t,i} \indic(A_t \in \cI)
- 
\epsilon \sum_{i \in \supp(\theta) \setminus \hat{S}}
\sum_{t=1}^n A_{t,i} \indic(A_t \in \cI)
\\
& =  
\epsilon \sum_{i \in \supp(\theta) \cap \hat{S}}
\del{ n - \sum_{t=1}^n A_{t,i} \indic(A_t \in \cI) }
+
\epsilon \sum_{i \in \supp(\theta) \setminus \hat{S}}
\del{ n - \sum_{t=1}^n A_{t,i} \indic(A_t \in \cI) }
\\
& \geq 
\epsilon \sum_{i \in \supp(\theta) \setminus \hat{S}}
\del{ n - \sum_{t=1}^n A_{t,i} \indic(A_t \in \cI) }
\\
& \geq 
\epsilon \sum_{i \in \supp(\theta) \setminus \hat{S}}
\del{ n - \frac{n}{2} }
\\
& \geq \epsilon \cdot \frac{s}{4} \cdot \frac{n}{2} = \frac{\epsilon s n}{8}.
\end{align*}
where the second equality is by decomposing $1 = \indic(A_t \in \cI) + \indic(A_t \in \cH)$; 
the first inequality is because for all $a \in \cH$ and all $\theta \in \Theta_s \cup \Theta_{2s}$, $\inner{\theta}{a} \leq 2s \kappa - 1 \leq 0$; 
\chicheng{Sorry I forgot -- Is this constraint mentioned in (or can be derived by) the assumptions of theorem statement?}
\ja{I will add this assumption if our main theorem does not include it.}

the third equality is by noting that $\inner{\theta}{A_t} = \epsilon \sum_{i \in \supp(\theta)} A_{t,i}$ and decomposing $\supp(\theta)$ to disjoint union $\supp(\theta) \cap \hat{S}$ and $\supp(\theta) \setminus \hat{S}$; 
the fourth equality is by algebra; the second inequality is by observing that $n - \sum_{t=1}^n A_{t,i} \indic(A_t \in \cI) \geq 0$ for all $i$; 
the third inequality is by noting that $\sum_{i=1}^d \sum_{t=1}^n |A_{t,i}| \indic(A_t \in \cI) \leq s n$, and therefore, for all $i \notin \hat{S}$, $\sum_{t=1}^n A_{t,i} \indic(A_t \in \cI) 
 \leq \sum_{t=1}^n |A_{t,i}| \indic(A_t \in \cI) \leq \frac{s n}{d / 2} \leq \frac{n}{2}$; the last two steps are by algebra.
}

{Given the above claim, we lower bound the minimax regret as follows:
\begin{itemize}
\item If there exists a $\theta\in \Theta_s \cup \Theta_{2s}$ that satisfies $\PP_\th( |M_\th(\hS)| \ge s/4) \ge \frac{1}{4}$, then by the claim above, $\EE_\theta[\Reg_n] \geq \frac{\epsilon s n}{8} \cdot \frac{1}{4} = \frac{\epsilon s n}{32}$. 

\item Otherwise, by Proposition~\ref{claim: main lower bound}, there exists $\theta' \in \Theta_{2s}$ such that $\EE_{\theta'}
[T(\cH)] \geq \Omega\del{ \frac{1}{ \kappa^2 \epsilon^2 } }$. 
Note that $\max_{a \in \cA} \inner{\theta}{a} - \max_{a \in \mathcal{H}} \inner{\theta}{a} \geq s\epsilon - (\kappa s \epsilon -1) \geq \Omega(1)$, this implies that $\EE_{\theta'}[\Reg_n] \geq \Omega( \frac{1}{\kappa^2 \epsilon^2} )$. 
\end{itemize}
In summary, for any symmetric augmented algorithm $\B$, there exists some $\theta \in \Theta_s \cup \Theta_{2s}$ such that 
\[
\EE_\theta [ \Reg_n ] \geq \Omega\del{ \min\del{ \epsilon s n, \frac{1}{\kappa^2 \epsilon^2} } }. 
\]
The theorem is concluded by recalling the choice of $\epsilon = \kappa^{-2/3} s^{-1/3} n^{-1/3}$.  
}
\end{proof}




\subsubsection{Proof of Proposition~\ref{claim: main lower bound}}\label{subsubsec: proof of main prop}

{By assumption,} $\forall \th \in \Theta_s \cup \Theta_{2s}, \PP_\th( \abs{ M_\th (\hS) } \ge s/4) \le \xi$. Define 
\begin{align*}
    \theta' &= (\underbrace{\epsilon, \cdots, \epsilon}_{s}, 0, \cdots, 0, -1)\in \Theta_s,\\
    \theta &= (\underbrace{\epsilon, \cdots, \epsilon}_{2s}, 0, \cdots, 0, -1)\in \Theta_{2s},\\
    \tilth &= \theta-\theta' .
\end{align*}
{We aim to show that for this $\theta'$, $\EE_{\th'} [T(\cH) ] \ge \Omega(\frac{1}{\kappa^2\epsilon^2})$.}

The following two lemmas show the main advantage why we set $\theta'$ and $\theta$ in this way. 

\begin{lemma}\label{lemma: symmetry}
Let $\phi \in \RR^d$, $\pi \in \Pi$, $E_1, E_2 \in \mathsf{Sub}_{d/2}$ be the elements which satisfies $\pi(\phi)=\phi$ and $\pi (E_1) = E_2$. Let $\tau \geq 0$. For any symmetric {augmented} algorithm $\B$, we have that $$\PP_{\B,\phi} (\hat{S}=E_1, T(\cH)\leq \tau)=\PP_{\B,\phi} (\hat{S}=E_2, T(\cH)\leq \tau).$$
\end{lemma}
\begin{proof}
To see this, note that:
\begin{align*}
     & \PP_{\B,\phi} (\hat{S}=E_1, T(\cH)\leq \tau) \\
     = & \EE_{\B, \phi} \sbr{ \ind \del{ \hat{S} = E_1, \sum_{t=1}^n \ind(A_t \in \cH) \leq \tau } } \tag{Definition of expectation} \\
    = & \EE_{\B, \pi(\phi)} \sbr{ \ind \del{ \pi^{-1}(\hat{S}) = E_1, \sum_{t=1}^n \ind( \pi^{-1}(A_t) \in \cH) \leq \tau } } \tag{Lemma~\ref{lem:algas}} \\
    = & \EE_{\B, \pi(\phi)} \sbr{ \ind \del{ \hat{S} = \pi(E_1), T(\cH) \leq \tau } } \tag{algebra}\\
    = & \EE_{\B, \phi} \sbr{ \ind \del{ \hat{S} = E_2, T(\cH) \leq \tau } } \tag{Definition of expectation, $\pi(\phi) = \phi$, $\pi(E_1) = E_2$}
\end{align*}
\end{proof}

Let $Q_{m} := \{ S \in \mathsf{Sub}_{d/2}| M_\th(S) =m\}$ and $Q_{m}'= \{ S \in \mathsf{Sub}_{d/2}| M_{\theta'}(S) =m\}$ be the collections of size $d/2$ sets which has exactly $m$ mistakes with {respect to} $\theta$ and $\theta'$, respectively.

\begin{lemma}\label{lemma: symmetry_E1E2 specific}
{Suppose $\theta, \theta'$ are defined above; fix $m$}.
\begin{itemize}
\item For {any} $E_1, E_2 \in Q_{m}$, there exists $\pi$ such that $\pi(\theta) = \theta$  (i.e. $\pi \in \{\pi_1 \circ \pi_2 | \pi_1 \in \Pi_{1:2s}, \pi_2 \in \Pi_{2s+1:d} \}$) which satisfies $\pi(E_1)=E_2$. 
\item Similarly, for {any} $E_1', E_2' \in Q_{m}'$, there exists $\pi$ such that $\pi(\theta') = \theta'$ (i.e. 
$\pi \in \{\pi_1 \circ \pi_2 | \pi_1 \in \Pi_{1:s}, \pi_2 \in \Pi_{s+1:d} \}$) which satisfies $\pi(E_1')=E_2'$.
\end{itemize}
\end{lemma}
\begin{proof}
Since $|E_1 \cap [1:2s]| = |E_2 \cap [1:2s]|$ and $|E_1 \cap [2s+1:d]| = |E_2 \cap [2s+1:d]|$, there exists $\pi \in \{\pi_1 \circ \pi_2 | \pi_1 \in \Pi_{1:2s}, \pi_2 \in \Pi_{2s+1:d} \}$ which satisfies $\pi(E_1)=E_2$. Similar proof holds also for $E_1' , E_2'$. 
\end{proof}
With foresight, define $\tau = 4 \EE_{\theta'}[T(\cH)]$. Then, by assumption, 
\chicheng{Can this $\frac{3}{4}s- l$ be replaced with $\frac{s}{2} + l$ throughout the proof? Accordingly, $\frac{1}{4}s - l$ can be replaced with $l$. (new $l \in \cbr{0, 1, \ldots, \frac s 4 -1}$)
}
\begin{align*}
  \xi \geq & \PP_\th\del{ \abs{M_\th(\hS)} \ge \frac{s}{4} }
  \\
  \geq &
  \PP_\th\del{ \abs{ M_\th(\hS) } \ge \frac{s}{4}, T(\cH) \le \tau} \\
  \geq & \sum_{l=1}^{\fr14 s} \PP_\th\del{ \abs{ M_\th(\hS) } = \fr34 s - l, T(\cH) \le \tau}
\end{align*}
For $l \in\{1,\ldots,s/4\}$, and let $R_{l} = Q_{\fr34 s-l} \cap Q_{\fr14 s-l}'$. For $a \in \mathbb{R}^d$ and $r\in \mathbb{R}$, 
and $\phi \in \cbr{\theta, \theta'}$,
let $p_\phi (r|a) = \frac{1}{\sqrt{2\pi}} \exp\del{-\frac{(r-\langle a, \phi \rangle)^2}{2}}$ be the probability density function of the reward $r$ when the action $a$ is given under hypothesis $\phi$. 
{Note that for any interaction history $(A_t, r_t)_{t=1}^T$, its probability density function under $\theta$ and $\theta'$ has ratio $\prod_{t=1}^n \frac{p_\th (r_t|A_t) }{p_{\th'}(r_t|A_t)}$. 
}

Now pick one element $E \in R_l$. 
Then, 
\begin{align}
  &\PP_\th\del{ \abs{ M_\th(\hS) } = \fr{3}{4}s -l, T(\cH) \le \tau }
  \\&= \PP_\th\del{ \hS \in Q_{\fr{3}{4}s-l}, T(\cH)\leq \tau }
  \\&= |Q_{\fr{3}{4}s-l}|\PP_\th\del{ \hS =E, T(\cH)\leq \tau } \nonumber\tag{Lemma \ref{lemma: symmetry} and \ref{lemma: symmetry_E1E2 specific}}
\\&= \frac{|Q_{\fr{3}{4}s  - l}|}{|R_l|} \PP_\th\del{\hS \in R_l, T(\cH) \le \tau} \nonumber\tag{Lemma \ref{lemma: symmetry} and \ref{lemma: symmetry_E1E2 specific}}
\\&= \frac{|Q_{\fr{3}{4}s  - l}|}{|R_l|} \EE_{\th'}\sbr{ \one(\hS \in R_l, T(\cH) \le \tau)\prod_{t=1}^n \frac{p_\th (r_t|A_t) }{p_{\th'}(r_t|A_t)} } \nonumber\tag{change of measure}
\\&= \frac{|Q_{\fr{3}{4}s  - l}|}{|R_l|} \EE_{\th'}\sbr{ \one(\hS \in R_l, T(\cH) \le \tau)\exp(-\sum_{t=1}^n \ln\frac{p_{\th'} (r_t|A_t) }{p_{\th}(r_t|A_t)}) } \label{eqn: where claim 1 is going}
\end{align}
{To proceed,} we will use the following {claim} to bound {the probability ratio}. {The proof of this claim requires our novel application of symmetry property of the algorithm $\B$ and is one of our key technical contributions; 
we defer its proof to Section~\ref{sec:claim-kl}}.

\begin{claim}\label{claim: KL divergence}
{For any $\rho, \delta > 0$,}
\begin{align*}
    & \EE_{\th'}\sbr{ \one(\hS \in R_l, T(\cH) \le \tau) \exp(-\sum_{t=1}^n \ln\frac{p_{\th'} (r_t|A_t) }{p_{\th}(r_t|A_t)}) } \\
    \geq & \PP_{\th'} \del{ \hS \in R_l, T(\cH) \le \tau } \exp(-KL(\rho,\delta,\tau))-\delta^{1+\frac{1}{\rho}}
\end{align*}
where $KL(\rho,\delta,\tau)=\frac{1}{2}\epsilon^2 (1+\rho)(\GIANTKLtau)+\frac{1}{\rho} \ln \frac1{\delta}$.
\end{claim}



Now, decide $\rho$ and $\delta$ later and continuing from the previous inequality with $KL(\rho,\delta,\tau)=\frac{1}{2}\epsilon^2 (1+\rho)(\GIANTKLtau)+\frac{1}{\rho} \ln \frac1{\delta}$, 
\begin{align*}
(\ref{eqn: where claim 1 is going})&\ge \frac{|Q_{\fr{3}{4}s  - l}|}{|R_l|} \del{\PP_{\th'} \del{ \hS \in R_l, T(\cH) \le \tau} \exp(-KL(\rho,\delta,\tau))-\delta^{1+\frac{1}{\rho}}} \tag{Claim \ref{claim: KL divergence}}
\\
&= |Q_{\fr{3}{4}s  - l}| \PP_{\th'} \del{ \hS = E, T(\cH) \le \tau} \exp(-KL(\rho,\delta,\tau))- \frac{|Q_{\fr 3 4 s- l}|}{|R_l|}\delta^{1+\frac{1}{\rho}}  \tag{Lemma \ref{lemma: symmetry} and \ref{lemma: symmetry_E1E2 specific}}
\\
&=
\fr{|Q_{\fr{3}{4}s  -l }|}{|Q_{\fr s 4 - l }'|} \PP_{\th'} \del{ \hS \in Q_{\fr s 4 - l}', T(\cH) \le \tau} \exp(-KL(\rho,\delta,\tau)) -\frac{|Q_{\fr 3 4 s- l}|}{|R_l|}\delta^{1+\frac{1}{\rho}}\tag{Lemma \ref{lemma: symmetry} and \ref{lemma: symmetry_E1E2 specific}}
\\
&= 
\fr{|Q_{\fr{3}{4}s  -l }|}{|Q_{\fr s 4 - l }'|} \del{\PP_{\th'} \del{ \abs{ M_\thp(\hS) } =\frac{s}{4}-l , T(\cH) \le \tau } \exp(-KL(\rho,\delta,\tau))-\frac{|Q_{\fr s 4 - l}'|}{|R_l|}\delta^{1+\frac{1}{\rho}}}
\\
&\geq \fr{|Q_{\fr{3}{4}s  -l }|}{|Q_{\fr s 4 - l }'|} \del{\PP_{\th'} \del{ \abs{ M_\thp(\hS) } =\frac{s}{4}-l , T(\cH) \le \tau } \exp(-KL(\rho,\delta,\tau))-s\delta^{1+\frac{1}{\rho}}} \tag{Lemma \ref{R vs Q'}}
\end{align*}
For the last inequality, we used the following lemma. 
\begin{lemma}\label{R vs Q'}
For $d>(s+1)^2$, $s>5$, and $l \in [\frac{s}{4}]$, we have $\frac{|Q'_{\fr{s}{4}-l}|}{|R_l|} < s$.
\chicheng{Can we merge this with the lemma below, which says that 
$\frac{|Q_{\fr{3s}{4}  - l}|}{|Q_{\fr s 4 - l }'|} \geq \exp(\Omega(s))$?
}
\end{lemma}
In short,
\begin{align*}
  \PP_\th(M_\th (\hS)= \fr34 s -l, T(\cH) \le \tau ) 
  &\ge \frac{|Q_{\fr{3s}{4}  - l}|}{|Q_{\fr s 4 - l }'|}  \del{\PP_{\th'}(M_\thp (\hS) = \fr s 4 - l, T(\cH) \le \tau) \exp( -KL(\rho,\delta,\tau)) -s\delta^{1+\frac{1}{\rho}}}~.
\end{align*}
Let $Y=\min_{l \in [\fr{s}{4}]} \frac{|Q_{\fr{3s}{4}  - l}|}{|Q_{\fr s 4 - l }'|}$. Then, 
\begin{align*}
  \PP_\th(M_\th(\hS) = \fr34 s -l, T(\cH) \le \tau ) 
  &\ge Y \del{\PP_{\th'}(M_\thp(\hS) = \fr s 4 - l, T(\cH) \le \tau) \exp( -KL(\rho,\delta,\tau))-s\delta^{1+\frac{1}{\rho}}}~.
\end{align*}
Summing up both sides for $l \in [\frac{s}{4}]$,
\begin{align*}
  \xi&\ge \sum_{l=1}^{\fr{s}{4}}\PP_\th(M = \fr34 s - l, T(\cH) \le \tau) 
\\ &\ge Y \del{\PP_{\th'}(M \le \fr s 4 - 1, T(\cH) \le \tau) \exp(-KL(\rho,\delta,\tau))-\fr{s^2}{4}\delta^{1+\frac{1}{\rho}}}
\\ &\ge Y \del{\PP_{\th'}(M \le \fr s 4 - 1) - \fr{\EE_{\th'}[T(\cH)]}{\tau} }  \exp(-KL(\rho,\delta,\tau)) - \frac{s^2 Y}{4} \delta^{1+\frac{1}{\rho}} \tag{$\PP(A,B) \ge \PP(A) - \PP(\bar B)$; Markov's ineq.}
\\ &\ge Y \del{1-\xi- \fr{\EE_{\th'}[T(\cH)]}{\tau} }  \exp(-KL(\rho,\delta,\tau))- \frac{s^2 Y}{4} \delta^{1+\frac{1}{\rho}}
\\ &= Y \del{1-2\xi}  \exp(-KL(\rho,\delta,\EE_{\th'}[T(\cH)]/\xi))- \frac{s^2 Y}{4} \delta^{1+\frac{1}{\rho}} \tag{set $\tau = \EE_{\th'}[T(\cH)]/\xi$}\\
&\geq \fr{Y}{2}\exp(-KL(\rho,\delta,\EE_{\th'}[T(\cH)]/\xi))- \frac{s^2 Y}{4} \delta^{1+\frac{1}{\rho}} \tag{setting $\xi\leq \fr{1}{4}$}
\end{align*}
Setting $\delta=(\frac{4\xi}{s^2 Y})^{\frac{\rho}{\rho+1}}$, $\rho=3$ and rearranging the last equation with sufficiently large $s$ we get:
\begin{align*}
  \EE_{\th'}[T(\cH)] \ge \frac{\xi}{77s\kappa^2} \del{\frac{1}{4\epsilon^2} \ln \frac{Y}{4\xi} - \frac{4ns^3 }{d}} \tag{Lemma \ref{claim: scale of Y}}
\end{align*}

\chicheng{
I don't know if the following presentation is cleaner; but for what it is worth: 
starting with 
\[ 
\frac{1}{4} \geq \frac{1}{4} Y ( \exp(-KL(\rho, 4\EE_\thp\sbr{T(\cH)})) \delta^{\frac1\rho} - s^2 \delta^{1 + \frac1\rho}  )
\]
where $KL(\rho, \tau) = \frac{(1+\rho)}{2} \cdot \epsilon^2 \cdot \del{{\fr{4s^3n}{d}+  77 s \kap^2 \tau}}$. 
Setting $\rho = 1$, and 
$\delta = \frac1{2s^2} \exp(-KL(\rho, 4\EE_\thp\sbr{T(\cH)}))$, we get 
\[
1 \geq \frac{Y}{s^2} \cdot \exp(-2 KL(\rho, 4\EE_\thp\sbr{T(\cH)}))
\]
which also implies that 
\[
\fr{4s^3n}{d}+ 4 \cdot 77 s \kap^2 \EE_{\thp}\sbr{ T(\cH) }
\geq 
\frac{1}{2}\ln\frac{Y}{s^2}
\]
and therefore, 
\[
\EE_{\thp}\sbr{ T(\cH) } \geq \frac{1}{4 \cdot 77 s \kap^2} \del{ \frac1{2\epsilon^2} \ln\frac{Y}{s^2} - \frac{4 n s^3}{d} }
\]
}

{Recall that in the construction, $\epsilon=\kappa^{-2/3}s^{-1/3}n^{-1/3}$; with this choice of $\epsilon$ and our assumption that $d \geq 
 \kappa^{-4/3}s^{4/3}n^{1/3}$, we have 
 $\frac{4 n s^3}{d} \leq \frac{s}{256 \epsilon^2}$. 
 On the other hand, $\frac{1}{2\epsilon^2} \ln\frac{Y}{s^2} \geq \frac{s}{128 \epsilon^2}$. Combining the above, we conlcude that 
 $\EE_{\th'}[T(\cH)]\geq \Omega(\frac{1}{\kappa^2 \epsilon^2})$.
}

and using the following Lemma \ref{claim: scale of Y} leads the conclusion that the order of $\EE_{\th'}[T(\cH)]\geq \Omega(\frac{1}{\kappa^2 \epsilon^2})$.

\begin{lemma}\label{claim: scale of Y}
When $s\geq 500$, $\ln Y \geq \frac{s}{72}$, and $\ln \frac{s^2 Y}{4\xi} \leq 2 \ln \frac{Y}{4\xi}$.
\end{lemma}

\subsubsection{Proof of Claim \ref{claim: KL divergence}}
\label{sec:claim-kl}

{The proof of} This claim consists of two parts - {first, we prove that with high probability, the probability ratio between the two hypotheses is controlled in terms of the ``empirical KL divergence'' $\sum_{t=1}^n \inner{\theta - \theta'}{A_t}^2$ (Lemma~\ref{lem: emp KL to KL - bigger}). 
}, and {second, we use symmetry to upper bound the negative exponential of empirical KL divergence.}. {We start with the first part.} 

\chicheng{Removed the ``uniform over $T$'' quantifier to make the statement simpler.}
\begin{lemma}\label{lem: emp KL to KL - bigger} 
For $\rho>0$, let $$B_{\theta', \theta}(\rho):= \cbr{ \sum_{t=1}^{n}\ln (\frac{p_{\theta'} (r_t|A_t)}{p_{\theta}(r_t|A_t)}) \geq (1+\rho)\sum_{t=1}^n \inner{\theta - \theta'}{A_t}^2 + \frac{1}{\rho}\ln \frac 1 \delta}$$ 
Then, $\PP_{\theta'}(B_{\theta, \theta'}(\rho))\leq \delta$.
\end{lemma}

\begin{proof}
Let 
$\cB_t = \sigma\cbr{A_1, r_1, \ldots, A_t, r_t, A_{t+1}}$ be the $\sigma$-field of all observations upto time step $t$ and the action at time step $t+1$. 

\begin{equation}
\blue{J_t} :=\ln (\frac{p_{\theta'} (r_t|A_t)}{p_{\theta}(r_t|A_t)}) = \frac{(r_t - \inner{\theta}{A_t})^2}{2} - \frac{(r_t - \inner{\theta'}{A_t})^2}{2}
= 
\eta_t \inner{\theta - \theta'}{A_t} + \frac{\inner{\theta' - \theta}{A_t}^2}{2} 
. 
\label{eqn:j-t}
\end{equation}
Then 
\begin{align}
    \EE_{\theta'} \sbr{ J_t \mid \cB_{t-1}} &= \frac{\inner{\theta' - \theta}{A_t}^2}{2} .
\label{eqn:j-t-exp}
\end{align}
Therefore, 
\[
\blue{J_t} - \EE_{\theta'} \sbr{ J_t \mid \cB_{t-1}} 
= \eta_t \inner{\theta - \theta'}{A_t}. 
\]
As a consequence, for any $\rho \in \RR$, 
\[
\EE_{\theta'} \sbr{ \exp(\rho( \blue{J_t} - \EE_{\theta', t-1} \sbr{ J_t \mid \cB_{t-1}}  )) \mid \cB_{t-1} }
= \exp(\rho^2  \EE_{\theta'} \sbr{ J_t \mid \cB_{t-1}});
\]
in other words, 
\[
\EE_{\theta'} \sbr{ \exp(\rho( \blue{J_t} - (1+\rho)\EE_{\theta', t-1} \sbr{ J_t \mid \cB_{t-1}}  )) \mid \cB_{t-1} } = 1.
\]



Now let ${H_t} = \exp(\rho (\sum_{s=1}^t J_s - (1+\rho)\EE_{\thp}[J_s|\cB_{s-1}]) $ with $H_0=1$. It can be seen that under $\theta'$,  $\cbr{H_t}_{t=1}^T$ is a nonnegative supermartingale with respect to filtration $\cbr{\cB_t}_{t=1}^T$; indeed, 
\begin{align*}
    \EE_{\theta'} [H_t \mid \cB_{t-1}] 
    &= \EE_{\theta'}\sbr{ \exp(\rho( \sum_{s=1}^t J_s - (1+\rho)\EE_{\thp}[J_s|\cB_{s-1}])) \mid \cB_{t-1}} \\
    &= H_{t-1} \EE_{\theta'}\sbr{ \exp(\rho ( J_t - (1+\rho)\EE_{\thp}[J_t|\cB_{t-1}]) ) \mid \cB_{t-1} }
    \\
    & = H_{t-1}.
\end{align*}

Finally, using Markov's inequality on $H_n$, we have
\[
\PP\del{ \exp\del{ \rho \cdot \sum_{t=1}^n \del{ J_t - (1+\rho)\EE_{\thp}[J_t|\cB_{t-1}] } } \geq \frac{1}{\delta} } \leq \delta. 
\]
The lemma is concluded by plugging in~\eqref{eqn:j-t} and~\eqref{eqn:j-t-exp} and algebra, and using the assumption that $\rho > 0$. 
\end{proof}

By the above Lemma \ref{lem: emp KL to KL - bigger}, one can {deduce} the following relationship. 


\begin{align}
    & \EE_{\th'}\sbr{ \one(\hS \in R_l, T(\cH) \le \tau)\exp(-\sum_{t=1}^n \ln\frac{p_{\th'} (r_t|a_t) }{p_{\th}(r_t|a_t)}) }
    \\
    &=
    \EE_{\th'} \sbr{ \one\del{\hS \in R_l, T(\cH) \le \tau, B_{\theta, \theta'}(\rho)^c} \exp\del{ -\frac{1}{2}(1+\rho)\sum_{t=1}^n \inner{A_t}{\tilth}^2-\frac{1}{\rho} \ln \frac1{\delta} } } \tag{Lemma \ref{lem: emp KL to KL - bigger}}
    \\
    &=
    \underbrace{\EE_{\th'}\sbr{\one\del{\hS \in R_l, T(\cH) \le \tau, B_{\theta, \theta'}(\rho)^c}\exp\del{-\frac{1}{2}(1+\rho)\sum_{t=1}^n \inner{A_t}{\tilth}^2 }}}_{\text{(X)}}\exp(-\frac{1}{\rho} \ln \frac1{\delta})
    \label{eqn:intro-x}
\end{align}

For the remaining part of the proof, we lower bound (X).
\begin{align}
\text{(X)} = & \EE_{\th'}\sbr{ \one\del{\hS \in R_l, T(\cH) \le \tau} 
\del{ 1 - \one\del{B_{\theta, \theta'}(\rho)} } \exp\del{-\frac{1}{2}
(1+\rho)\sum_{t=1}^n \inner{A_t}{\tilth}^2 } } \\
\geq & \EE_{\th'}\sbr{ \one\del{\hS \in R_l, T(\cH) \le \tau }\exp\del{-\frac{1}{2}(1+\rho)\sum_{t=1}^n \inner{A_t}{\tilth}^2 } } - \delta
\label{eqn:minus-delta}
\end{align}

Define $\Pi = \cbr{ \pi_1 \circ \pi_2: \pi_1 \in \sym([1:s]), \pi_2 \in \sym([s+1:d]) }$. 
\chicheng{I don't think we need to vary $\pi_1$ in $\Pi$'s definition. (7/19) I take it back. To allow any element in $Q_a'$ to be mapped to any other element in $Q_a'$ using $\Pi$, we need to have the first part $\pi_1$.}
Importantly, for any $\pi \in \Pi$, $\pi(\theta') = \theta'$.

We focus on the first term in the above expression:
\begin{align}
  & \EE_{\th'}\sbr{ \one\del{\hS \in R_l, T(\cH) \le \tau }\exp\del{-\frac{1}{2}(1+\rho)\sum_{t=1}^n \inner{A_t}{\tilth}^2 } } \nonumber\\
  = & 
  \frac{1}{|\Pi|} \sum_{\sigma \in \Pi} \EE_{\sigma(\th')}\sbr{ \one\del{\sigma^{-1}(\hS) \in R_l} \cdot
  \one\del{\sum_{t=1}^n \indic(\sigma^{-1}(A_t) \in \cH) \le \tau } \exp\del{-\frac{1}{2}(1+\rho)\sum_{t=1}^n \inner{\sigma^{-1}(A_t)}{\tilth}^2 } } \tag{Lemma~\ref{lem:algas}}\nonumber\\
  = & 
  \EE_{\th'}\sbr{ \del{ 
  \frac{1}{|\Pi|} \sum_{\sigma \in \Pi} \one\del{\sigma^{-1}(\hS) \in R_l }\exp\del{-\frac{1}{2}(1+\rho)\sum_{t=1}^n \inner{\sigma^{-1}(A_t)}{\tilth}^2 } } \cdot \indic\del{ T(\cH) \le \tau } },
  \label{eqn:symmetrization-unified}
\end{align}
where the second equality uses the fact that for any $\sigma \in \Pi$, $\sigma(\theta') = \theta'$, and $\sum_{t=1}^n \indic(\sigma^{-1}(A_t) \in \cH) = \sum_{t=1}^n \indic(A_t \in \cH) = T(\cH)$. 

Now for any realization of $\bA, \hat{S}$, namely, $\ba = (a_1, \ldots, a_n)\in \cA^n, \hatsr \in \mathsf{Sub}_{d/2}$, we lower bound the quantity 
\[
\frac{1}{|\Pi|} \sum_{\sigma \in \Pi} \one\del{\sigma^{-1}(\hatsr) \in R_l }\exp\del{-\frac{1}{2}(1+\rho)\sum_{t=1}^n \inner{\sigma^{-1}(a_t)}{\tilth}^2 }
\]
in the following claim. 

\begin{claim}
For any set $\hatsr \in \mathsf{Sub}_{d/2}$, and any $a_1, \ldots, a_n$,
\begin{align*}
& \frac{1}{|\Pi|} \sum_{\sigma \in \Pi} \one\del{\sigma^{-1}(\hatsr) \in R_l }\exp\del{-\frac{1}{2}(1+\rho)\sum_{t=1}^n \inner{\sigma^{-1}(a_t)}{\tilth}^2 }\\
\geq &
\indic\del{ \hatsr \in Q'_{\frac14 s - l} }
\cdot 
\frac{|R_l|}{|Q'_{\frac14 s - l}|} \cdot
\exp\del{ -\frac12 (1+\rho)( \GIANTKL) \epsilon^2 }
\end{align*}
\end{claim}
\begin{proof}
If $\hatsr \notin Q_{\frac14 s - l}'$, then for any permutation $\sigma \in \Pi$, it must be the case that $\sigma^{-1}(\hatsr) \notin Q'_{\frac14 s - l}$, and therefore, $\sigma^{-1}(\hatsr) \notin R_l$. In this case, both sides are equal to zero and the claim is trivially true.

Otherwise, $\hatsr \in Q'_{\frac14 s - l}$. 
Define $\Pi_{\legal}(l, \hatsr) = \cbr{ \sigma \in \Pi: \sigma^{-1}(\hatsr) \in R_l }$. 
Using this notation, the left hand side can be equivalently written as:
\begin{align*}
& \frac{1}{|\Pi|} \sum_{\sigma \in \Pi_{\legal}(l,\hatsr)} \exp\del{-\frac{1}{2}(1+\rho)\sum_{t=1}^n \inner{\sigma^{-1}(a_t)}{\tilth}^2 } \\
= &
\frac{|\Pi_{\legal}(l,\hatsr)|}{|\Pi|} \cdot \frac{1}{|\Pi_{\legal}(l,\hatsr)|} \sum_{\sigma \in \Pi_{\legal} (l, \hatsr)} \exp\del{-\frac{1}{2}(1+\rho)\sum_{t=1}^n \inner{\sigma^{-1}(a_t)}{\tilth}^2 } \\
\geq & 
\frac{|\Pi_{\legal} (l, \hatsr)|}{|\Pi|} \cdot  \exp\del{-\frac{1}{2}(1+\rho) \del{  \frac{1}{|\Pi_{\legal} (l, \hatsr)|} \sum_{\sigma \in \Pi_{\legal} (l, \hatsr)} \sum_{t=1}^n \inner{\sigma^{-1}(a_t)}{\tilth}^2 } } \tag{Jensen}\\
\geq & 
\frac{|\Pi_{\legal} (l, \hatsr)|}{|\Pi|} \cdot
\exp\del{ -\frac12 (1+\rho)( \GIANTKL) \epsilon^2 } \tag{Lemma \ref{lemma: KL rand to Const}}\\
= & 
\frac{|R_l|}{|Q'_{\frac14 s - l}|} \cdot
\exp\del{ -\frac12(1+\rho) (\GIANTKL) \epsilon^2 } \tag{Lemma \ref{lemma:ratio of pipi and RQ}}
\end{align*}
Here, the last two steps rely on Lemmas~\ref{lemma: KL rand to Const} and~\ref{lemma:ratio of pipi and RQ} respectively; we defer their statements and proofs to the end of this section.
\end{proof}


We now continue Equation~\eqref{eqn:symmetrization-unified} to conclude the proof of Claim~\ref{claim: KL divergence}:
let $E$ be an arbitrary element of $R_l$; 

\begin{align*}
  & \eqref{eqn:symmetrization-unified} \\
  \geq & 
    \EE_{\th'}\sbr{ \indic\del{ \hat{S} \in Q'_{\frac14 s - l}, T(\cH) \le \tau }
     }  
     \cdot 
    \frac{|R_l|}{|Q'_{\frac14 s - l}|} \cdot
    \exp\del{ -\frac12 (1+\rho)( \GIANTKLtau) \epsilon^2 } \\
    = & 
    \EE_{\th'}\sbr{ \indic\del{ \hat{S} = E, T(\cH) \le \tau }
     }  
     \cdot 
    |R_l| \cdot
    \exp\del{ -\frac12 (1+\rho)( \GIANTKLtau) \epsilon^2 } \\
    = & 
    \EE_{\th'}\sbr{ \indic\del{ \hat{S} \in R_l, T(\cH) \le \tau }
     }  
     \cdot
    \exp\del{ -\frac12 (1+\rho)( \GIANTKLtau) \epsilon^2 }
\end{align*}
Claim~\ref{claim: KL divergence} is now concluded by plugging this inequality back to Eq.~\eqref{eqn:minus-delta}, and back to Eq.~\eqref{eqn:intro-x}. 


\paragraph{Deferred lemmas and proofs.} For the remainder of this subsection, we present the statements of Lemmas~\ref{lemma:ratio of pipi and RQ} and~\ref{lemma: KL rand to Const} along with their proofs.


\begin{lemma}\label{lemma:ratio of pipi and RQ}
For any $u \in Q'_{\frac14 s-l}$, 
\[
\frac{|\Pi_{\legal} (l, \hatsr)|}{|\Pi|} = \frac{|R_l|}{|Q'_{\frac14 s - l}|}
\]
\end{lemma}
\begin{proof}
It suffices to prove $|\Pi_{\legal} (l, \hatsr)||Q'_{\frac14 s - l}| = |R_l| |\Pi|$. 
To see this, first note that for any $\hat{u} \in Q_{\frac{1}{4}s - l}'$, 
there is some $\pi \in \Pi$ such that $\pi(u) = \hat{u}$, and thus 
$|\Pi_{\legal}(l, \hat{u})| = |\Pi_{\legal}(l,\hatsr)|$.
Next, 
note,
\begin{align*}
    \text{(LHS)} = & \sum_{\hat{\hatsr} \in Q'_{\frac14 s - l}} |\Pi_{\legal}(l, \hat{u})|  \\
    & = \sum_{\hat{\hatsr} \in Q'_{\frac14 s - l}} \sum_{\sigma \in \Pi} \indic(\sigma^{-1}(\hat{\hatsr}) \in R_l) \tag{Definition of $\Pi_{\legal, u}$} \\
    & = \sum_{\sigma \in \Pi} \sum_{\hat{\hatsr} \in Q'_{\frac14 s - l}}  \indic(\sigma^{-1}(\hat{\hatsr}) \in R_l) \tag{algebra} \\
    & = \sum_{\sigma \in \Pi} |R_l| \tag{$\sigma$ induces 1-1 mapping over sets, and $\sigma^{-1}(\hat{\hatsr}) \in R_l \implies \hat{u} \in Q_{\frac14 s - l}'$} \\
    & = \text{(RHS)}.
\end{align*}
\end{proof}
\begin{remarks}
The above lemma can also be seen by noting that for any $v \in Q_{\fr 1 4 s - l}'$, 
\begin{equation}
\sum_{\sigma \in \Pi} \indic(\sigma^{-1}(u) = v) = \sum_{\sigma \in \Pi} \indic(\sigma^{-1}(u) = u) = \sum_{\sigma \in \Pi} \indic(\sigma(u) = u) = \frac{|\Pi|}{|Q_{\fr 1 4 s - l}'|},
\label{eqn:orbit-stab}
\end{equation}
where the last equality is by the orbit-stablizer theorem (consider group $\Pi$ acting on sets in $\mathsf{Sub}_{d/2}$; $\sum_{\sigma \in \Pi} \indic(\sigma(u) = u)$ is the size of the stabilizer subgroup of $u$, and $Q_{\fr 1 4 s - l}'$ is the orbit of $u$). 
Summing Eq.~\eqref{eqn:orbit-stab} over all $v \in R_l$ yields $|\Pi_{\legal}(l, \hat{u})| = |R_l| \cdot \frac{|\Pi|}{|Q_{\frac14 s - l}'|}$, hence the lemma.
\end{remarks}


\begin{lemma}\label{lemma: KL rand to Const}
Assume $s \le \sqrt{d} $ and $d \ge 16$.
  For any $\ba \in \cA^n$ and $u \in Q'_{\fr14s - l}$,
  \begin{align}
    \frac{1}{|\Pi_{\legal}(l,u)|} \sum_{t=1}^n \sum_{\sigma \in \Pi_{\legal}(l,u)} \la \sig^{-1}(a_t), \tilth\ra^2 
    \le  \del{\GIANTKL}\eps^2
    \label{eqn:kl-rand-to-const}
  \end{align}
\end{lemma}

\begin{proof}

{
We will show the following claim:
\begin{numcases} 
{\frac{1}{|\Pi_{\legal}(l,u)|} \sum_{\sigma \in \Pi_{\legal}(l,u)} \la \sig^{-1}(a), \tilth\ra^2 \leq}
\frac{4 s^3}{d} \epsilon^2, & $a \in \cI$, 
\label{eqn:perm-i}
\\
77 s \kappa^2  \epsilon^2, & $a \in \cH$.
\label{eqn:perm-h}
\end{numcases}
The lemma follows 
from this claim by noting that the LHS of Eq.~\eqref{eqn:kl-rand-to-const} can be decomposed to 
\[
 \sum_{t: a_t \in \cI} \frac{1}{|\Pi_{\legal}(l,u)|} \sum_{\sigma \in \Pi_{\legal}(l,u)} \la \sig^{-1}(a_t), \tilth\ra^2 
 +
 \sum_{t: a_t \in \cH} \frac{1}{|\Pi_{\legal}(l,u)|} \sum_{\sigma \in \Pi_{\legal}(l,u)} \la \sig^{-1}(a_t), \tilth\ra^2 
\]
and we apply the above claim to the two terms respectively. 

We first prove Eq.~\eqref{eqn:perm-i}. For $a \in \cI$, 
\begin{align*}
&\frac{1}{|\Pi_{\legal}(l,u)|} \sum_{\sigma \in \Pi_{\legal}(l,u)} \la \sig^{-1}(a), \tilth\ra^2
\\
= &  \eps^2 \cdot \frac{1}{|\Pi_{\legal}(l,u)|} \sum_{\sigma \in \Pi_{\legal}(l,u)} (\sum_{j=s+1}^{2s} |a_{\sig(j)}|)^2
\\
\le & s \eps^2 \cdot \frac{1}{|\Pi_{\legal}(l,u)|}  \sum_{j=s+1}^{2s} \sum_{\sigma \in \Pi_{\legal}(l,u)} |a_{\sig(j)}|
\\
= & s \eps^2 \frac{1}{|\Pi_{\legal}(l,u)|} \sum_{j=s+1}^{2s} \sum_{\sigma \in \Pi_{\legal}(l,u)} \sum_{h=s+1}^{d} |a_{h}|  \indic\del{\sig(j) = h} 
\\
= & s\eps^2 \sum_{j=s+1}^{2s} \sum_{h=s+1}^{d} |a_h| \cd  \frac{1}{|\Pi_{\legal}(l,u)|} \sum_{\sigma \in \Pi_{\legal}(l,u)}   \indic\del{\sigma(j) = h} 
\\
\leq & s\eps^2 \sum_{j=s+1}^{2s} \sum_{h=s+1}^{d} |a_h| \cdot \frac{4}{d} 
\\
\leq & \frac{4 s^3 \epsilon^2}{d}
\end{align*}
where the first inequality is by Cauchy, the second inequality is due to item 1 of Lemma~\ref{lem:legal-perm}. \ja{TODO: Change the constant $\frac{4}{d}$ to something larger, and check the constant dependency. - done... maybe too large I guess?}

}

{We next prove Eq.~\eqref{eqn:perm-h}. For $a \in \cH$,
\begin{align*}
&\frac{1}{|\Pi_{\legal}(l,u)|}  \sum_{\sigma \in \Pi_{\legal}(l,u)} \la \sig^{-1}(a), \tilth\ra^2
\\
=&   
\epsilon^2 \cdot \frac{1}{|\Pi_{\legal}(l,u)|}  \sum_{\sigma \in \Pi_{\legal}(l,u)} \del{ \sum_{j=s+1}^{2s} a_{\sigma(j)} }^2
\\
=& \eps^2 \cdot \fr{1}{|\Pi_{\legal}(l,u)|} \del{ s \kappa^2 + \sum_{a,b \in [s+1:2s]: a\neq b} a_{\sig(a) } a_{\sig(b) } } 
\\
=& \eps^2 s \kap^2 + \epsilon^2 \cd   
  \sum_{i,j \in [s+1:d]: i \neq j}  a_{i} a_{j} \del{ \sum_{a,b \in [s+1:2s]: a\neq b}
  \fr{1}{|\Pi_{\legal}(l,u)|} \sum_{\sigma \in \Pi_{\legal}(l,u)}  \indic( \sigma(a) = i, \sigma(b) = j ) }
\\
=& \eps^2 s \kap^2 + \epsilon^2 s (s-1) \cd   
  \underbrace{ \sum_{i,j \in [s+1:d]: i \neq j}  a_{i} a_{j} \del{ 
  \fr{1}{|\Pi_{\legal}(l,u)|} \sum_{\sigma \in \Pi_{\legal}(l,u)}  \indic( \sigma(s+1) = i, \sigma(s+2) = j ) } }_{(Z_2')}
\end{align*}
Here, the last equality follows from Lemma~\ref{lem:legal-perm} that $\fr{1}{|\Pi_{\legal}(l,u)|} \sum_{\sigma \in \Pi_{\legal}(l,u)}  \indic( \sigma(a) = i, \sigma(b) = j )$ are all equal across all $(a, b)$'s (and we choose $a = s+1$, $b = s+2$ without loss of generality).
In the sequel, define 
\[
f(i,j) = \fr{1}{|\Pi_{\legal}(l,u)|} \sum_{\sigma \in \Pi_{\legal}(l,u)}  \indic( \sigma(s+1) = i, \sigma(s+2) = j ).
\]

To compute $(Z'_2)$, let us define the following where the first two are false positives w.r.t. $\supp(\th')$ and the last two are true negatives w.r.t. $\supp(\th')$:
\begin{align*}
  \blue{C_+} &=\{i \in [s+1:d] \cap u | a_{i}'=\kappa \}\\
  \blue{C_-} &=\{i \in [s+1:d] \cap u | a_{i}'=-\kappa \}\\
  \blue{M_+} &=\{i \in [s+1:d] \setminus u | a_{i}'=\kappa \}\\
  \blue{M_-} &=\{i \in [s+1:d] \setminus u | a_{i}'=-\kappa \}~.
\end{align*}
Let $\blue{c_+, c_-, m_+, m_-}$ be $|C_+|, |C_-|, |M_+|, |M_-|$, respectively. Note that $c_+ + c_- = \frac{d}{2}-\frac{3}{4}s - l = k$ and $m_+ + m_- = \frac{d}{2}-\frac{1}{4}s + l = d - s - k$. 

$Z_2'$ can be decomposed to the following three major cases (which consists of subcases), depending on $i$ and $j$ lying in the one of these four sets:
\begin{itemize}
\item $i, j \in u$: common coefficient $f(i,j) = \frac{\frac{s}{2} (\frac{s}{2} - 1) }{s(s-1) k(k-1)} =: f_1$. 
\begin{itemize}
\item $i \in C_+, j \in C_+$: $c_+(c_+ - 1) \cdot f_1 \cdot \kappa^2$ 
\item $i \in C_+, j \in C_-$ or $i \in C_-, j \in C_+$: $- 2 c_+ c_- \cdot f_1 \cdot \kappa^2$ 
\item $i \in C_-, j \in C_-$: $c_-(c_- - 1) \cdot f_1 \cdot \kappa^2$ 
\end{itemize}
Summary: total contribution 
$\kap^2\fr{\fr s 2 \del{\fr s 2 - 1} }{s ( s - 1 )} \cdot
      \del{  \fr{ c_+(c_+-1) + c_-(c_- -1) - 2 c_+ c_-}{k(k-1)}}$

\item $i, j \in u^c$: common coefficient $f(i,j) = \frac{\frac{s}{2} (\frac{s}{2} - 1) }{(d-s-k)(d-s-k-1) k(k-1)} =: f_2$
\begin{itemize}
\item $i \in M_+, j \in M_+$: $m_+(m_+ - 1) \cdot f_2 \cdot \kappa^2$ 
\item $i \in M_+, j \in M_-$ or $i \in M_-, j \in M_+$: $- 2 m_+ m_- \cdot f_2 \cdot \kappa^2$ 
\item $i \in M_-, j \in M_-$: $m_-(m_- - 1) \cdot f_2 \cdot \kappa^2$ 
\end{itemize}
Summary: total contribution 
$\kap^2\fr{\fr s 2 \del{\fr s 2 - 1} }{s ( s - 1 )} \cdot
      \del{  \fr{ m_+(m_+-1) + m_-(m_- -1) - 2 m_+ m_-}{(d-s-k)(d-s-k-1)}}$

\item exactly one of $i,j$ is in $u$ (the other is in $u^c$): $f(i,j) = \fr{\fr s 2 \del{\fr s 2 - 1} }{s ( s - 1 ) k(d-s-k)} =: f_3$
\begin{itemize}
\item $i \in M_+, j \in C_+$ or $i \in C_+, j \in M_+$: $2 m_+ c_+   \cdot f_3 \cdot \kap^2$
\item $i \in M_-, j \in C_-$ or $i \in C_-, j \in M_-$: $2 m_- c_-  \cdot f_3 \cdot \kap^2$
\item $i \in M_-, j \in C_+$ or $i \in C_+, j \in M_-$: $- 2 m_- c_+ \cdot f_3 \cdot \kap^2$
\item $i \in M_+, j \in C_-$ or $i \in C_-, j \in M_+$: $- 2 m_+ c_- \cdot f_3 \cdot \kap^2$
\end{itemize}
Summary: total contribution $2 \kap^2\fr{\fr s 2 \cd \fr s 2 }{s ( s - 1 )} 
      \del{ \fr{(c_+ - c_-)(m_+ - m_-)}{k(d-s-k)} }$
\end{itemize}

}

Then,
\begin{align*}
  (Z'_2)
  &\le \kap^2\fr{\fr s 2 \del{\fr s 2 - 1} }{s ( s - 1 )} 
      \del{  \fr{ c_+(c_+-1) + c_-(c_- -1) - 2 c_+ c_-}{k(k-1)} 
          + \fr{ m_+(m_+-1) + m_-(m_- -1) - 2 m_+ m_-}{(d-s-k)(d-s-k-1)}
    } 
  \\ &\quad + \kap^2\fr{\fr s 2 \cd \fr s 2 }{s ( s - 1 )} 
      \del{ 2\cd \fr{(c_+ - c_-)(m_+ - m_-)}{k(d-s-k)}  } 
\\&\le \kap^2\fr{\fr s 2 \del{\fr s 2 - 1} }{s ( s - 1 )} 
      \del{  \fr{ (c_+ - c_-)^2 }{k(k-1)} 
          + \fr{ (m_+ - m_-)^2}{(d-s-k)(d-s-k-1)}
    } 
  \\ &\quad + \kap^2\fr{\fr s 2 \cd \fr s 2 }{s ( s - 1 )} 
      \del{ 2\cd \fr{(c_+ - c_-)(m_+ - m_-)}{k(d-s-k)}  } \tag{$c_+ + c_- \ge 0, m_+ - m_- \ge 0$ }
\\&\le \kap^2\fr{\fr s 2 \cd \fr s 2 }{s ( s - 1 )} 
      \del{  \fr{ (c_+ - c_-)^2 }{k(k-1)} 
          + \fr{ (m_+ - m_-)^2}{(d-s-k)(d-s-k-1)}
          + 2\cd \fr{(c_+ - c_-)(m_+ - m_-)}{k(d-s-k)}
    } 
\\&= \underbrace{\kap^2\fr{\fr s 2 \cd \fr s 2 }{s ( s - 1 )} 
     \cd \del{\fr{c_+ - c_-}{k} + \fr{m_+ - m_-}{d-s-k} }^2}_{\tsty =: (Z'_{2,1})}
  +  \underbrace{\kap^2\fr{\fr s 2 \cd \fr s 2 }{s ( s - 1 )} \del{ \fr{ (c_+ - c_-)^2 }{k^2(k-1)} 
      + \fr{ (m_+ - m_-)^2}{(d-s-k)^2(d-s-k-1)}
    }}_{\tsty =: (Z'_{2,2})} 
      \tag{$\fr{1}{x(x-1)} = \fr{1}{x^2} + \fr{1}{x^2(x-1)}   $}
\end{align*}

For $(Z'_{2,1})$, let $w=\frac{(c_+ +m_+)-(c_- + m_-)}{2}$.
Note 
\[
|2w| = { \abs{ \cbr{ i \in [s+1:d]: a_i = +\kappa } } - \abs{ \cbr{i \in [s+1:d]: a_i = -\kappa } } } \leq \sqrt{2d \ln (2d)}+s.
\]
Then, 
from definition of $w$, $c_+ - c_- = 2w - (m_+ - m_-)$.
So, 

\chicheng{
Here is a perhaps simpler derivation that help convey the intuition.
\begin{align*}
& \evt {\fr{c_+ - c_-}{k} + \fr{m_+ - m_-}{d-s-k} }
\\
\leq &  
\abs {\fr{c_+ - c_-}{d/2} + \fr{m_+ - m_-}{d/2} } + 
\abs {\fr{c_+ - c_-}{d/2} - \fr{c_+ - c_-}{k} } +
\abs {\fr{m_+ - m_-}{d/2} - \fr{m_+ - m_-}{d-s-k} }
\\
\leq & 
\abs{ \frac{w}{d/2} } + \abs{c_+ - c_-} \frac{\frac34 s + l}{d/2 \cdot k} + \abs{m_+ - m_-} \frac{\frac14s - l}{d/2 \cdot (d-s-k)} 
\\
\leq & 
\frac{2\sqrt{2d \ln(2d)} + s}{d} + \frac{4s}{d} + \frac{4s}{d} 
\end{align*}
where the last step uses that $k \geq \frac d 4$ and $d - s - k \geq \frac d 4$. 
} \ja{I totally agree that it is cleaner, but I'm a bit afraid of changing all 77s...}
\chicheng{No worries at all!}

\begin{align*}
  \fr{c_+ - c_-}{k} + \fr{m_+ - m_-}{d-s-k}  
  &= \fr{2w - (m_+ - m_-)}{k} + \fr{m_+ - m_-}{d - s - k}  
\\&= \fr{1}{k} \del{2w - (\fr{d-s-2k}{d-s-k} ) (m_+ - m_-) } 
\\&= \fr{1}{k} \del{2w - (\fr{2 l - s}{d-s-k} ) (m_+ - m_-) } 
\\
\implies 
\evt {\fr{c_+ - c_-}{k} + \fr{m_+ - m_-}{d-s-k} }
  &\le \fr{1}{k} \del{2w + \evt{\fr{2 l - s}{d-s-k}} (d - s - k) } 
\\  &\le \fr{1}{k} \del{\sqrt{2d \ln (2d)}+s + s } 
    ~\le \fr{12\sqrt{\ln (2d)}}{\sqrt{d}} \tag{$k \ge \frac{d}{2\sqrt{2}}$, $\sqrt{2d \ln (2d)}>s$}
\end{align*}

Thus, using $s\ge 2$,
\begin{align*}
  (Z'_{2,1}) &= \kap^2\fr{\fr s 2 \cd \fr s 2 }{s ( s - 1 )} \cd \fr{72 \ln (2d)}{d} 
             ~\le \kap^2\cd\fr{36 \ln (2d)}{d} 
\end{align*}

For $(Z'_{2,2})$, using $k \wed (d-s-k) \ge \fr d 4$,
\begin{align*}
  (Z'_{2,2}) &= \kap^2\cd \fr{\fr s 2 \cd \fr s 2 }{s ( s - 1 )} \cd \del{\fr{1}{k-1} + \fr{1}{d-s-k-1} } 
\\  &\le \kap^2 \cd \fr12 \cd \fr{8}{d} = \frac{4\kap^2}{d} < \frac{4\kap^2 \ln (2d)}{d}
\end{align*}

Thus, $  (Z'_2) \le \frac{76\kap^2 \ln (2d)}{d}$.

Altogether,
\begin{align*}
  &\frac{1}{|\Pi_{\legal}(l,u)|} \sum_{t: a_t \in \cH} \sum_{\sigma \in \Pi_{\legal}(l,u)} \la \sig^{-1}(a_t), \tilth\ra^2
  \\&\le \eps^2 \cd \del{ s \kap^2 \cT(\cH; a)  
    + \cT(\cH; a) \cd s(s-1) \cd 76\frac{\kap^2 \ln (2d)}{d}  
  }
  \\&< 77 \eps^2 s \kap^2 \cT(\cH; a) \tag{$(s-1) \ln (2d) < d$}
\end{align*}
where the last inequality is by $(s-1) < d$.
\end{proof}

\chicheng{In the $\Pi$ below, it does not vary $\pi_1 \in \sym[1:s]$. 
(7/19) We need to vary $\pi_1 \in \sym[1:s]$ because of the complication raised above, so I revised it. The statement of the lemma does not change; it is just 
START OF LEMMA 13}
{
\begin{lemma} 
\label{lem:legal-perm}
We have: 
\begin{enumerate}
\item Let $j \in [s+1:2s]$ and $h \in [s+1:d]$, 
\begin{equation}
\frac{1}{|\Pi_{\legal}(l,u)|} \sum_{\sigma \in \Pi_{\legal}(l,u)}   \indic\del{\sigma(j) = h} 
=
\begin{cases}
\frac{1}{2k} & h \in u \\
\frac{1}{2(d-s-k)} & h \notin u 
\end{cases}
\label{eqn:single-frac}
\end{equation}
\item Let $a,b$ be distinct elements of $[s+1:2s]$ and $i,j$ be distinct elements of $[s+1:d]$, 
\begin{equation}
\frac{1}{|\Pi_{\legal}(l,u)|} \sum_{\sigma \in \Pi_{\legal}(l,u)}   \indic\del{\sigma(a) = i, \sigma(b) = j} 
= 
\begin{cases}
\frac{\frac{s}{2} (\frac{s}{2} - 1) }{s(s-1) k(k-1)} & i \in u, j \in u, 
\\
\frac{ (\frac{s}{2})^2 }{s(s-1) k(d-s-k)} & i \in u, j \notin u \text{ or }  i \notin u, j \in u
\\
\frac{\frac{s}{2} (\frac{s}{2} - 1) }{s(s-1) (d-s-k)(d-s-k-1)} & i \notin u, j \notin u, 
\end{cases}
\label{eqn:pairwise-frac}
\end{equation}
\end{enumerate}
\end{lemma}
\begin{proof}
Let $\tilde{\Pi}_{\legal}(l,u) = \cbr{\pi_2 \in \sym([s+1:d]): \pi_2^{-1}(u) \in R_l}$.
Note that 
$\Pi_{\legal}(l,u) = \cbr{\pi_1 \circ \pi_2: \pi_1 \in \sym([1:s]), \pi_2 \in \tilde{\Pi}_{\legal}(l,u)}$. 

Therefore, the left hand side of Eq.~\eqref{eqn:single-frac} can be simplified as:
\[
\frac{
|\cbr{ \pi_1 \circ \pi_2: \pi_1 \in \sym([1:s], \pi_2 \in \tilde{\Pi}_{\legal}(l,u), \pi_2(j) = h }|
}{ |\sym([1:s])| \times |\tilde{\Pi}_{\legal}(l,u)|}
= 
\frac{1}{|\tilde{\Pi}_{\legal}(l,u)|}
\sum_{\pi_2 \in \tilde{\Pi}_{\legal}(l,u)}
\indic\del{\sigma(j) = h}.
\]

Similarly, the left hand side of Eq.~\eqref{eqn:pairwise-frac} can be simplified as:
\begin{align*}
    \frac{
|\cbr{ \pi_1 \circ \pi_2: \pi_1 \in \sym([1:s], \pi_2 \in \tilde{\Pi}_{\legal}(l,u), \pi_2(a) = i, \pi_2(b) = j }|
}{ |\sym([1:s])| \times |\tilde{\Pi}_{\legal}(l,u)|}= 
\frac{1}{|\tilde{\Pi}_{\legal}(l,u)|}
\sum_{\pi_2 \in \tilde{\Pi}_{\legal}(l,u)}
\indic\del{\pi_2(a) = i, \pi_2(b) = j}.
\end{align*}


With the above simplifications, we now use the following equivalent formulation to guide our calculation. Let elements in $u$ represent distinct red balls (there are $k$ of them), and elements in $u^c \cap [s+1:d]$ represent distinct black balls (there are $d-s-k$ of them). 
Denote by $[s+1:2s]$ and $[2s+1:d]$ bin 1 and bin 2, respectively.
We call each coordinate of a bin a \emph{slot}, and all slots are distinct. 

$|\tilde{\Pi}_{\legal}(l,u)|$ equals the number of arrangements of all $d-s$ balls, such that bin 1 has exactly $\frac{s}{2}$ red balls and $\frac{s}{2}$ black balls, which is equal to 
\[
\underbrace{ {k \choose {\frac{s}{2}}} }_{\text{red balls in bin 1}} \cdot \underbrace{ {d-s-k \choose {\frac{s}{2}}} }_{\text{black balls in bin 2}} \cdot \underbrace{ s! }_{\text{bin 1 arrangement}} \cdot \underbrace{ (d-2s)! }_{\text{bin 2 arrangement}}
\]
\begin{enumerate}
\item We now calculate $\sum_{\sigma \in \tilde{\Pi}_{\legal}(l,u)}   \indic\del{\sigma(j) = h}$:
\begin{itemize}
\item When $h \in u$, this is equal to the number of arrangements of all $d-k$ balls, such that bin 1 has exactly $\frac{s}{2}$ red balls and $\frac{s}{2}$ black balls, and the $j$-th slot contains a specific red ball $h$. The number of such arrangements is 
\[
\underbrace{ {k - 1 \choose {\frac{s}{2}} - 1} }_{\text{remaining red balls in bin 1}} \cdot \underbrace{ {d-s-k \choose {\frac{s}{2}}} }_{\text{black balls in bin 2}} \cdot \underbrace{ (s-1)! }_{\text{bin 1 arrangement}} \cdot \underbrace{ (d-2s)! }_{\text{bin 2 arrangement}}
\]

\item When $h \in u$, this is equal to the number of arrangements of all $d-k$ balls, such that bin 1 has exactly $\frac{s}{2}$ red balls and $\frac{s}{2}$ black balls, and the $j$-th slot contains a specific black ball $h$. The number of such arrangements is 
\[
\underbrace{ {k \choose {\frac{s}{2}}} }_{\text{remaining red balls in bin 1}} \cdot \underbrace{ {d-s-k - 1 \choose {\frac{s}{2} - 1} } }_{\text{black balls in bin 2}} \cdot \underbrace{ (s-1)! }_{\text{bin 1 arrangement}} \cdot \underbrace{ (d-2s)! }_{\text{bin 2 arrangement}}
\]
\end{itemize}
The item is obtained by dividing the respective counts by the expression of $|\tilde{\Pi}_{\legal}(l,u)|$, along with algebra.

\item We now calculate $\frac{1}{|\tilde{\Pi}_{\legal}(l,u)|} \sum_{\sigma \in \tilde{\Pi}_{\legal}(l,u)}   \indic\del{\sigma(a) = i, \sigma(b) = j}$:
\begin{itemize}
\item When $i \in u$, $j \in u$, this is equal to the number of arrangements of all $d-k$ balls, such that bin 1 has exactly $\frac{s}{2}$ red balls and $\frac{s}{2}$ black balls, and the $a$-th slot contains a specific red ball $i$, and the $b$-th slot contains a specific red ball $j$. The number of such arrangements is 
\[
\underbrace{ {k - 2 \choose {\frac{s}{2}} - 2} }_{\text{remaining red balls in bin 1}} \cdot \underbrace{ {d-s-k \choose {\frac{s}{2}}} }_{\text{black balls in bin 2}} \cdot \underbrace{ (s-2)! }_{\text{bin 1 arrangement}} \cdot \underbrace{ (d-2s)! }_{\text{bin 2 arrangement}}
\]
\item When $i \in u$, $j \notin u$, this is equal to the number of arrangements of all $d-k$ balls, such that bin 1 has exactly $\frac{s}{2}$ red balls and $\frac{s}{2}$ black balls, and the $a$-th slot contains a specific red ball $i$, and the $b$-th slot contains a specific black ball $j$. The number of such arrangements is 
\[
\underbrace{ {k - 1 \choose {\frac{s}{2}} - 1} }_{\text{remaining red balls in bin 1}} \cdot \underbrace{ {d-s-k-1 \choose {\frac{s}{2} - 1}} }_{\text{black balls in bin 2}} \cdot \underbrace{ (s-2)! }_{\text{bin 1 arrangement}} \cdot \underbrace{ (d-2s)! }_{\text{bin 2 arrangement}}
\]
The same calculation goes through when $i \notin u$, $j \in u$.
\item When $i \notin u$, $j \notin u$, this is equal to the number of arrangements of all $d-k$ balls, such that bin 1 has exactly $\frac{s}{2}$ red balls and $\frac{s}{2}$ black balls, and the $a$-th slot contains a specific black ball $i$, and the the $b$-th slot contains a specific black ball $j$. The number of such arrangements is 
\[
\underbrace{ {k \choose {\frac{s}{2}} } }_{\text{remaining red balls in bin 1}} \cdot \underbrace{ {d-s-k - 2 \choose {\frac{s}{2}} - 2} }_{\text{black balls in bin 2}} \cdot \underbrace{ (s-2)! }_{\text{bin 1 arrangement}} \cdot \underbrace{ (d-2s)! }_{\text{bin 2 arrangement}}
\]
\end{itemize}

\end{enumerate}

\end{proof}
}

\subsubsection{Proofs of Lemma \ref{R vs Q'} and Lemma \ref{claim: scale of Y}}

{Before we proceed, we will prove the following two lemmas to compute ratio of combinations.}
The following lemma gives constant-factor tight bounds on the binomial coefficient ${s \choose i}$ when the number of successes $i$ is about the half of the total number of trials $s$. 
\begin{lemma}
\label{lem:near-central}
For even number $s \geq 16$ and $t \in [\frac s 2 - \frac{\sqrt{s}}{2}, \frac s 2 + \frac{\sqrt{s}}{2}]$,
\[
\frac{1}{2e} \frac{1}{\sqrt{s}} 2^s  \leq  {s \choose t} \leq \frac{1}{\sqrt{s}} 2^s
\]
\end{lemma}
\begin{proof}
First, it is known from~\cite{tikhonov2020comparative} that 
\[
\frac{2^s}{2\sqrt{s}}
\leq
(1-\frac{1}{4s}) \frac{2^s}{\sqrt{\pi s / 2}} \leq {s \choose {\frac s 2}} \leq (1-\frac{1}{4.5s}) \frac{2^s}{\sqrt{\pi s / 2}}
\leq 
\frac{2^s}{\sqrt{s}}
\]

Next, we claim that for all $t \in [\frac{s}{2} - \sqrt{s}, \frac{s}{2} + \sqrt{s}]$, $\frac{1}{e^4}{s \choose {\frac s 2}} \leq  {s \choose t} \leq {s \choose {\frac s 2}}$. Note that this concludes the proof by combining with the bounds on ${s \choose {\frac s 2}}$ above. 

By symmetry, it suffices to show that
for $i \in [0, \sqrt{s}]$, 
\[
\frac 1 {e^4}  {s \choose {\frac s 2}} \leq {s \choose {\frac s 2 + i}} \leq {s \choose {\frac s 2}}. 
\]
Indeed, 
\[
\frac{{s \choose \frac s 2 + i}}{{s \choose {\frac s 2}}} = \prod_{j=1}^i \frac{\frac{s}{2}-i+j}{\frac{s}{2}+j}
\]
It is clear that the right hand side is at most 1 as each factor is $\leq 1$; on the other hand, as for each $j$,  $\frac{\frac{s}{2}-i+j}{\frac{s}{2}+j} \geq \frac{ \frac{s}{2} - i }{ \frac{s}{2} }$, we have
\[
\prod_{j=1}^i \frac{\frac{s}{2}-i+j}{\frac{s}{2}+j}
\geq 
(1 - \frac{2i}{s})^i 
\geq
e^{-\frac{4i^2}{s}}
\geq 
e^{-1}. 
\]
where the second inequality uses the fact that for $x \leq \frac{1}{2}$, $1-x \geq e^{-2x}$ and $\frac{2i}{s} \leq \frac{2}{\sqrt{s}} \leq \frac{1}{2}$.
\end{proof}

The following lemma gives an upper bound on the probability mass function of the hypergeometric distribution.
\begin{lemma}
\label{lem:hypergeom-binom}
For $a \leq c$ and $d \leq c$, $b \leq \min(a, d)$ such that
$\frac{b}{d} \geq \frac{a}{c}$,
\[
\frac{ {a \choose b} {c-a \choose d-b} }{ {c \choose d}  }
\leq 
\exp\del{ - 2 d \cdot \del{ \frac{b}{d} - \frac{a}{c} }^2 }.
\]
\end{lemma}
\begin{proof}
Note that by tail probability bounds of hypergeometric random variables~\cite{chvatal1979tail},  
\[
\sum_{i=b}^{d} \frac{ {a \choose i} {c-a \choose d-i} }{ {c \choose d} }
\leq 
\exp\del{ - 2 d \cdot \del{ \frac{b}{d} - \frac{a}{c} }^2 }.
\]
The follows from observing that the LHS is the probability of observing at least $b$ red balls by sampling $d$ balls without replacement from an urn of $a$ red balls and $c-a$ black balls.  
The lemma follows by observing that all terms on the left hand side are nonnegative. 
\end{proof}

\subsubsection{Proof of Lemma \ref{R vs Q'}}

\begin{proof}[Proof of Lemma \ref{R vs Q'}]
Recall 
$$|R_l|= {d-2s \choose \frac{d}{2}-\frac{3}{4}s + l}{s \choose \frac{s}{2}}{s \choose \frac{s}{4}-l},$$ $$|Q_{\fr14s-l}'| = {d-s \choose \frac{d}{2}-\frac{s}{4}+l} {s \choose \frac{1}{4}s - l}$$
The observation is that ${d-2s \choose \frac{d}{2}-\frac{3}{4}s + l}$, ${s \choose \frac{s}{2}}$ and ${d-s \choose \frac{d}{2}-\frac{s}{4}+l}$ are all near-central binomial coefficients. 
\chicheng{This imposes assumptions e.g. $\frac{1}{4}s - l \leq \frac{\sqrt{d-2s}}{2}$ -- is it satisfied by the construction?
} \ja{This is true because which is satisfied due to $d>s^2$ and $s>4$. }
Therefore we can use Lemma~\ref{lem:near-central}:
\[
\frac{|Q_{\fr14s-l}'|}{|R_l|}
=  
\frac{ {d-s \choose \frac{d}{2}-\frac{s}{4}+l}  }{ {d-2s \choose \frac{d}{2}-\frac{3}{4}s + l} {s \choose \frac{s}{2}} }
\leq 
\frac{
\frac{1}{\sqrt{d-s}} 2^{d-s}
}{
\frac{1}{2 e \sqrt{d-2s}} 2^{d-2s} \cdot \frac{1}{2 e \sqrt{s}} 2^s
}
\leq 
4 e^2 \sqrt{s}
\qedhere
\]
\end{proof}

\subsection{Proof of Lemma \ref{claim: scale of Y}}
\begin{proof}[Proof of Lemma \ref{claim: scale of Y}]

{Note that
\[
\frac{1}{Y} = 
\frac{|Q_{\frac14s-l}'|}{|R_l|} 
\frac{|R_l|}{|Q_{\frac34s-l}|}
\]
The former ratio, $\frac{|Q_{\frac14s-l}'|}{|R_l|} $, is at most $4 e^2 \sqrt{s}$ from Lemma \ref{R vs Q'}, and when $s\geq 1000$,  $4 e^2 \sqrt{s} \leq \exp(\frac{s}{72})$. 
For the latter, note that by Lemma~\ref{lem:hypergeom-binom},
\[
\frac{|R_l|}{|Q_{\frac34s-l}|}
= 
\frac{ {s \choose \frac{s}{2}} {s \choose \frac{s}{4}-l}}{ {2s \choose \frac{3s}{4} - l} }
\leq 
\exp\del{ -2 \del{\frac{3s}{4} - l} \del{ \frac{\frac s 2}{\frac{3s}{4} - l} - \frac{s}{2s} }^2 }
\leq 
\exp\del{ -\frac{s}{36} }
\]
where the last inequality is due to $\frac{3s}{4} - l \geq \frac{s}{2}$ and $\frac{\frac{s}{2}}{\frac{3s}{4} - l} \geq \frac{2s}{3}$. 
}
Therefore, $\frac{1}{Y} \leq \exp(\frac{s}{72})$ and $\ln Y \geq \frac{s}{72}$. 

To prove $\ln \frac{s^2 Y}{4\xi} \leq 2\ln \frac{Y}{4\xi}$, note that it is equivalent to prove $2 \ln s^2 \leq \ln Y$ (note that $\xi = \frac{1}{4}$). Since we already know $\ln Y \geq \frac{s}{72}$, it is enough to show $\frac{s}{72} \geq 2 \ln s$. This is true for $s\geq 1000$. 

\end{proof}

\subsubsection{Scale of $\kappa^2$ with respect to $C_{\min}$} \label{subsubsec:kappa vs cmin}
Last thing we have to deal with is connecting $\kappa^2$ to $\cC_{\min}$ and $H_*^2$. 
\begin{lemma}
$\cC_{\min} (\cH) \geq \frac{\kappa^2}{2}$
\end{lemma}

\begin{proof}
By the maximality of $\cC_{\min} (\cH)$, $\cC_{\min}(\cH)\geq \lambda_{min}(Q(\mathsf{Unif}(\cH)))$. If we prove that $\lambda_{min}(Q(\mathsf{Unif}(\cH))) \geq \frac{\kappa^2}{2}$, then the proof is done. 

Let's define $\cH_d$ as 
$$\cH_d = \cbr{ a \in \{ -\kappa, \kappa\}^d  ~\sVert[3]~ \envert{\sum_{i=1}^d a_i } \leq \kappa \sqrt{2d \ln (2d)}} $$
It is the set of first $d$ coordinate vectors of $\cH$. 

We can express $Q(\mathsf{Unif}(\cH))$ in terms of $Q(\mathsf{Unif}(\cH_d))$ as follows

\begin{align*}
    Q(\mathsf{Unif}(\cH)) = \begin{bmatrix}Q(\mathsf{Unif}(\cH_d)) & \vec{0}\\
    \vec{0}^\top & 1
    \end{bmatrix}
\end{align*} 

Therefore, the proof boils down to calculating $Q(\mathsf{Unif}(\cH_d))$. We can connect this matrix to $Q(\mathsf{Unif}(\{-\kappa, \kappa \}^d))$ by the following method
\begin{align*}
    Q(\mathsf{Unif}(\{-\kappa, \kappa \}^d)) &= Q(\mathsf{Unif}(\cH_d)) \times \PP_{A \sim \mathsf{Unif}(\cH_d)} (A \in \cH_d) \\&+Q(\mathsf{Unif}(\{-\kappa, \kappa \}^d \backslash \cH_d)) \times \PP_{A \sim \mathsf{Unif}(\{-\kappa, \kappa \}^d)} (A \notin \cH_d)
\end{align*}

Now, since Rademacher is 1 sub-Gaussian random variable, $\sum_{i=1}^n a_i$ is $\kappa \sqrt{d}$ sub-Gaussian random variable when $a \sim \mathsf{Unif}(\{ -\kappa , \kappa \})$. Therefore, 
\begin{align*}
    \PP_{A \sim \mathsf{Unif}(\{-\kappa, \kappa \}^d)} (A \notin \cH_d) &= \PP_{A \sim \mathsf{Unif}(\{-\kappa, \kappa \}^d)} (|\sum_{i=1}^d A_i |\geq \kappa \sqrt{2d \ln (2d)}) \leq \frac{1}{2d},
\end{align*}
where the last inequality is by the Hoeffding's inequality. Therefore, we can rewrite $Q(\mathsf{Unif}(\cH_d))$ as 
\begin{align*}
    Q(\mathsf{Unif}(\cH_d)) &\succeq \PP_{A \sim \mathsf{Unif}(\{-\kappa, \kappa \}^d)} (A \in \cH_d) Q(\mathsf{Unif}(\cH_d))\\
    &\succeq Q(\mathsf{Unif}(\{-\kappa, \kappa \}^d)) -Q(\mathsf{Unif}(\{-\kappa, \kappa \}^d \backslash \cH_d))\PP_{A \sim \mathsf{Unif}(\{-\kappa, \kappa \}^d)} (A \notin \cH_d)\\
    &\succeq Q(\mathsf{Unif}(\{-\kappa, \kappa \}^d)) -Q(\mathsf{Unif}(\{-\kappa, \kappa \}^d \backslash \cH_d))\frac{1}{2d}
\end{align*}
Therefore, (Existing results about this positive definite matrix analysis)
\begin{align*}
    \lambda_{\min} (Q(\mathsf{Unif}(\cH_d)))&\geq \lambda_{\min}(Q(\mathsf{Unif}(\{-\kappa, \kappa \}^d)) -Q(\mathsf{Unif}(\{-\kappa, \kappa \}^d \backslash \cH_d))\frac{1}{2d}) \\
    &\geq \lambda_{\min}(Q(\mathsf{Unif}(\{-\kappa, \kappa \}^d))) - \lambda_{\max} (Q(\mathsf{Unif}(\{-\kappa, \kappa \}^d \backslash \cH_d))\frac{1}{2d})
\end{align*}

Since every element in $\{ -\kappa, \kappa\}$ has $\ell_2$-norm $\sqrt{d}\kappa$, 

\begin{align*}
    \lambda_{\max} (Q(\mathsf{Unif}(\{-\kappa, \kappa \}^d \backslash \cH_d)) &= \max_{v \in \mathbb{S}^{d-1}} v^\top \sum_{a \in \{-\kappa, \kappa \}^d \backslash \cH_d} \frac{1}{|\{-\kappa, \kappa \}^d \backslash \cH_d|}a a^\top v \\&= \max_{v \in \mathbb{S}^{d-1}} \sum_{a \in \{-\kappa, \kappa \}^d \backslash \cH_d} \frac{1}{|\{-\kappa, \kappa \}^d \backslash \cH_d|} (a^\top v)^2 \\&\leq \max_{v \in \mathbb{S}^{d-1}} \sum_{a \in \{-\kappa, \kappa \}^d \backslash \cH_d} \frac{d\kappa^2}{|\{-\kappa, \kappa \}^d \backslash \cH_d|} = d\kappa^2
\end{align*}
and by simple symmetry one can calculate $\lambda_{\min} ( Q(\mathsf{Unif}(\{-\kappa, \kappa \}^d ))) = \kappa^2$. Therefore, $$\lambda_{\min} (Q(\mathsf{Unif}(\cH_d))) \geq \kappa^2 - d\kappa^2 \frac{1}{2d} = \frac{\kappa^2}{2}$$
\end{proof}

\section{Experiment details}
\begin{itemize}
    \item Case 1 - $\ell_1$ estimation error experiment
    \begin{itemize}
        \item $\theta= -e_1 +e_i$, $i \in \{2, \cdots, d\}$ chosen uniformly random before the start of the experiment.
        \item Dimension $d=10$, sparsity $s=2$
        \item Action set $\cA = \{e_1 + \frac{1}{\sqrt{d}}e_i | i =2, \cdots, d \} \cup \{ \frac{1}{\sqrt{d}} e_1\}$
        \item $T=1000, 2000, \cdots, 10000$
        \item $\sigma=0.1$
        \item Repetition: 30 times for each exploration time. 
    \end{itemize}
    \item Case 1 - bandit experiment
    \begin{itemize}
        \item $\theta= e_1 +e_i$, $i \in \{2, \cdots, d\}$ chosen uniformly random before the start of the experiment.
        \item Dimension $d=10$, sparsity $s=2$
        \item Action set $\cA = \{e_1 + \frac{1}{\sqrt{d}}e_i | i =2, \cdots, d \} \cup \{ \frac{1}{\sqrt{d}} e_1\}$
        \item $T=400000$
        \item $\sigma=0.1$
        \item Repetition: 30 times
    \end{itemize}
    
    \item Case 2 - $\ell_1$ estimation error experiment
    \begin{itemize}
        \item $\theta= e_i +e_j$, $i,j \in [d]$ chosen uniformly random before the start of the experiment.
        \item Dimension $d=30$, sparsity $s=2$
        \item Action set $\cA$: 90 Uniform random vectors over $\mathbb{S}^{d-1}$ before the start of the round, where $\mathbb{S}^{d-1} = \{ v \in \mathbb{R}^d | \|v\|_2 = 1\}$
        \item $T=1000, 2000, \cdots, 10000$
        \item $\sigma=0.1$
        \item Repetition: 30 times for each exploration time. 
    \end{itemize}
    \item Case 2 - bandit experiment
    \begin{itemize}
        \item $\theta= e_i +e_j$, $i,j \in [d]$ chosen uniformly random before the start of the experiment.
        \item Dimension $d=30$, sparsity $s=2$
        \item Action set $\cA$: 90 Uniform random vectors over $\mathbb{S}^{d-1}$ before the start of the round. 
        \item $T=10000$
        \item $\sigma=0.1$
        \item Repetition: 30 times for each exploration time. 
    \end{itemize}
\end{itemize}


\end{document}